\documentclass[11pt]{article}
\usepackage[sort,numbers]{natbib}
\usepackage[english]{babel}
\usepackage{fixltx2e,amsmath,amsfonts,amsthm}
\usepackage{graphicx}
\usepackage[mathscr]{euscript}
\usepackage{enumerate}
\usepackage{longtable}
\usepackage[T1]{fontenc}
\usepackage[table]{xcolor}
\usepackage{array,ragged2e}
\usepackage{authblk}
\usepackage{caption}
\usepackage{multirow}
\usepackage{dsfont}
\usepackage{amssymb}
\usepackage{float}
\usepackage{multirow}
\usepackage[shortlabels]{enumitem}
\usepackage{mathtools}
\usepackage{fancyhdr}
\usepackage[lined]{algorithm2e}

\MakeRobust{\eqref}

\newcommand*{\myfont}{\fontfamily{phv}\selectfont}
\DeclareTextFontCommand{\ggplotFont}{\myfont}

\DeclarePairedDelimiter{\ceil}{\lceil}{\rceil}

\usepackage[margin=.8in]{geometry}
\usepackage[font={footnotesize},margin=1cm]{caption}
\usepackage{bm}
\theoremstyle{definition}
\usepackage{enumitem}

\DeclareMathOperator*{\E}{\mathbb{E}}

\theoremstyle{remark}
\newtheorem{remark}{Remark}

\newcommand\eps{\ensuremath{\epsilon}}

\DeclareMathOperator*{\sign}{sign}

\DeclareMathOperator*{\argmax}{arg\,max}

\newtheorem{theorem}{Theorem}[section]
\newtheorem{lemma}[theorem]{Lemma}

\newtheorem{corollary}[theorem]{Corollary}

\SetAlCapSkip{1ex}
\SetNlSty{}{}{}

\title{Generalized Probabilistic Bisection for Stochastic Root-Finding}
\author[1]{Sergio Rodriguez}
\author[1]{Mike Ludkovski}
\affil[1]{University of California, Santa Barbara. Department of Statistics and Applied Probability.}

\begin{document}

\maketitle

\abstract{We consider numerical schemes for root finding of noisy responses through generalizing the Probabilistic Bisection Algorithm (PBA) to the more practical context where the sampling distribution is unknown and location-dependent. As in standard PBA, we rely on a knowledge state for the approximate posterior of the root location. To implement the corresponding Bayesian updating, we also carry out inference of oracle accuracy, namely learning the probability of correct response. To this end we utilize batched querying in combination with a variety of frequentist and Bayesian estimators based on majority vote, as well as the underlying functional responses, if available. For guiding sampling selection we investigate both Information Directed sampling, as well as Quantile sampling. Our numerical experiments show that these strategies perform quite differently; in particular we demonstrate the efficiency of randomized quantile sampling which is reminiscent of Thompson sampling. Our work is motivated by the root-finding sub-routine in pricing of Bermudan financial derivatives, illustrated in the last section of the paper.}

	\section{Introduction}\label{sec:intro}
	
	Finding the root of a function, i.e.,~solving the equation  $h(x^*) = 0$ for $x^*$, is a fundamental problem in numerical optimization.  In the stochastic setting the function $h(\cdot)$ is not observed directly but is learned via a stochastic simulator (aka \textit{oracle}). Applications of such Stochastic Root Finding Problems (SRFP) \cite{Pasupathy:2011:SRP:1921598.1921603} are relevant to many scientific disciplines. For instance, see quantile estimation for bio-assay experiments~\cite{finney1952statistical}, quality and reliability improvement~\cite{joseph2002operating}, sensitivity experiments~\cite{never1994ad}, and adaptive control and signal processing~\cite{chen2006stochastic}.
	
	An important special case, which is the focus of this paper, concerns one-dimensional root-finding where $x^* \in \mathbb{R}$ is known to be unique \cite{waeber2013probabilistic}. For example, this is the case if $h$ is monotone.
	Recently, this setting was also encountered in the context of simulation-based methods for optimal stopping, see~\cite{gramacy2013sequential}. Thus, we consider a data generating process of the form
	\begin{equation}
	\label{eq:pba_simulator}
	Z(x_{n}) = h(x_{n}) + \epsilon(x_{n}),
	\end{equation}
	where without loss of generality $h$ is negative to the left of $x^*$ and positive for all $x > x^*$, and where moreover $\epsilon$ is assumed to be a symmetric zero-mean  random component, dependent on $x_{n}$ but independent of previous evaluations.
	Due to noise, the actual $h(\cdot)$ in \eqref{eq:pba_simulator} is unknown. This leads to two fundamental paradigms for inference of $x^*$. The first strategy is to learn $h(\cdot)$, i.e.,~build a surrogate $\hat{h}$ and then take $\hat{x} = \hat{h}^{-1}(0)$ to be the root of $\hat{h}(\cdot)$, obtained via  a standard deterministic root-finding method (say Newton's method if $\hat{h}'$ is also available). This offers an opportunity to import the vast machinery of emulation/meta-modeling/surrogate construction which is an extensive topic in simulation optimization and computer experiment literatures \cite{brochu2010tutorial}. Within this context, root-finding is equivalent to contour-estimation, i.e.,~learning the boundary of $\{ x : h(x) > 0\}$, see \cite{RanjanBingham08,Picheny12,gramacy2013sequential}. Unfortunately, a ``good'' representation for $\hat{h}$ usually  does not lead to tractable models for $\hat{x}$. For example, with a Gaussian Process (GP) model for $h$, the marginal $\hat{h}(x)$ is Gaussian for any fixed $x$, however there is no closed-form expression for the distribution of $\hat{h}^{-1}(0)$. In fact, GP's, like most surrogate frameworks, would not lead in general to a unique root.
	
	The above limitation points to the alternative of modeling  $x^*$ itself, with $h(\cdot)$ as a background latent object. This corresponds to root-finding \cite{Pasupathy:2011:SRP:1921598.1921603}. One promising recent solution is the Probabilistic Bisection Algorithm (PBA)~\cite{waeber2013probabilistic} which explicitly built a Bayesian representation for the posterior distribution of $x^*$ given the collected data, and established exponential convergence of the posterior  to a Dirac mass at the true $x^*$. Such explicit performance guarantees are highly desirable and have not been available via the response surface modeling approaches.
	
	The goal of SRFP is to efficiently learn $x^*$, interpreted as optimizing the simulation budget of calling \eqref{eq:pba_simulator} by judiciously selecting the sampling points  $x_{1:n}: = (x_{1}, x_2,\ldots,x_{n})$ at which to query the oracle. This problem falls under the general rubric of \textit{experimental design}~\cite{chaloner1995Bayesian}.
	One sequential design framework, known as \textit{Stochastic Approximation}, is to choose $x_{1:n}$ such that $x_{n} \rightarrow x^{*}$ in probability, which closely resembles the Newton-Raphson method for nonlinear root-finding~\cite{robbins1951stochastic}. Another sequential method is to leverage the classical bisection search. Namely, to obtain information about the root an oracle is queried regarding whether $x^*$ lies to the left or to the right of a given point $x$. In a noise-free setting, this allows to halve the search space at each iteration. To account for noise \cite{horstein1963sequential} introduced the PBA, which updates a probability density based on the history of previous sampling locations and corresponding oracle responses utilizing Bayesian methods.  Waeber et al. extended the PBA to the stochastic setting taking $x^{*}$ as the realization of a continuous random variable $X^{*}$~\cite{waeber2013probabilistic}. Specifically, the stochastic PBA works with the \textit{sign} of the noisy function evaluations,
	\begin{equation}
	\label{eq:pba_response}
	Y(x_{n}) := \sign \{Z(x_{n})\}.
	\end{equation}
	Given the unicity of the root $x^{*}$, the oracle response \eqref{eq:pba_response} is equivalent to reporting the \textit{direction} where $x^{*}$ is located with respect to the querying point $x_{n}$. The noise in $Z$ from \eqref{eq:pba_simulator} then translates into potentially inaccurate oracle directions. Namely,  the oracle \eqref{eq:pba_response} produces a correct sign only with probability $p(x_{n})$.

	The main hurdle for practical application of PBA is the requirement to know the statistical properties of the stochastic sampler. In a realistic context, the \textit{specificity} of the oracle, namely $p(x_{n})$, is unknown and spatially varying in $x_{n}$ (since it should depend on $h(x_{n})$ in \eqref{eq:pba_simulator}).
	Without further assumptions, the only way to employ PBA is to \emph{estimate} $p(x_{n})$. This was already noted in \cite{waeber2013probabilistic} and more recently in \cite{frazier2016probabilistic}, who used a hypothesis-testing-inspired procedure to learn $p(x_{n})$ on route to learning $X^*$. Specifically, they employed a Test of Power One (TPO)~\cite{siegmund1985sequential}, which relies on repeated sampling of \eqref{eq:pba_response} to effectively boost $p(x_{n})$ to a fixed accuracy level $\tilde{p}(x_{n})$. However, such boosting can be very expensive in the regime that $p(x_{n}) = 0.5 + \delta$ for small $\delta$. This highlights the second challenge with PBA: in the context of root-finding and \eqref{eq:pba_simulator}, for $x_{n}$ close to $x^*$ we have $p(x_{n}) \simeq 0.5$ which implies that the oracle is \emph{uninformative} in the neighborhood of the root. A na\"{i}ve implementation of PBA leads to sampling too close to $x^*$ and is not asymptotically convergent (in the sense of the posterior collapsing to a Dirac mass at $x^{*}$).
	Finally, when the oracle behavior is itself unknown, Bayesian updating is necessarily heuristic, and special care must be taken in constructing and propagating the knowledge state that is proxy for the true posterior of $X^{*}$. Our experiments show that this behavior is very delicate and must be taken into account to obtain robust search strategies.
	
	In this article, we resolve the above challenges by providing \emph{practically-minded} extension of the PBA. We construct a class of algorithms, called generalized PBA (G-PBA) that can (i) efficiently learn oracle properties; and (ii) efficiently aggregate collected information to construct a sequential design. Similar to \cite{waeber2013probabilistic} we rely on \emph{batched sampling} to learn $p(x_{n})$; however in contrast to the former strategy of boosting, we directly work with a location-dependent oracle. This is shown to be more efficient, in particular by providing a better control on how much to batch.  For the sequential design piece, our main contribution is to propose several heuristics which blend the concept of Information Directed Sampling with the median-sampling rule proven to be optimal in \cite{waeber2013probabilistic} (under very restrictive assumptions). Namely, we investigate a class of quantile-sampling strategies. We also show the effectiveness of randomization which turns out to be crucial in preventing uncontrolled error propagation in constructing the knowledge state.
	
	Our G-PBA schemes are generic and can be employed across a wide spectrum of SRFP's. In that sense, they make minimal assumptions about the underlying \eqref{eq:pba_simulator}. To illustrate this robustness we use G-PBA to learn the critical exercise threshold in the context of Regression Monte Carlo for Optimal Stopping. In that case, the behavior of \eqref{eq:pba_simulator} is highly non-standard, in particular strongly heteroskedastic and non-Gaussian, which makes standard statistical learning procedures for $\hat{h}$ very sensitive~\cite{ludkovskiJCF}. In contrast, the G-PBA is rather agnostic to these challenges, not least thanks to the batching sub-steps which allow the Central Limit Theorem (CLT) smoothing of all statistical anomalies.

	As in standard PBA, we work in the sequential paradigm, selecting the next sampling point  $x_{n+1}$ given the existing sampling locations $x_{1:n}$ and corresponding responses $Z_{1:n} := (Z_{1}(x_{1}),\ldots,Z_{n}(x_{n}))$. The sampling framework is based on a Bayesian perspective captured via a \emph{knowledge state} $f_n$ which intuitively corresponds to the posterior distribution $p(X^* | x_{1:n}, z_{1:n})$. The  knowledge state is then used for the twin purposes of providing an estimate $\hat{x}_n$ of $X^{*}$, and for guiding the sequential design. Note the distinction between  \textit{sampling} and \textit{estimation}: the locations $x_{1:n}$ determine the $Z_{1:n}$'s, but need not be identical to the estimates $\hat{x}_{1:n}$ for the root $X^{*}$.
	
	To select $x_{n+1}$ we utilize two different approaches which both take advantage of the full probabilistic description of the root $X^{*}$ via the knowledge state $f_{n}$. Firstly, we borrow the idea of Expected Improvement (EI), popularized in Bayesian optimization. EI  constructs a one-step information gain criterion and sets $x_{n+1}$ as the corresponding greedy maximizer.  Examples of EI functions include Efficient Global Optimization (EGO) \cite{JonesSchonlauWelch98}, Stepwise Uncertainty Reduction (SUR) \cite{ChevalierPicheny13}, Expected Quantile Improvement (EQI) \cite{PichenyGinsbourger13}, and Integrated Mean Squared Error (IMSE) \cite{GramacyLee09}.  Here again we contrast the function-view strategy of emulation, which quantifies the learning of $h(\cdot)$, with the root-view strategy that quantifies learning of $X^*$. For the former, despite some progress on building EI measures for the level-sets and graph of $h(\cdot)$ \cite{ChevalierPicheny13,AzzimontiBect16}, these metrics remain complex.  In our view this is a fundamental conceptual hurdle arising from the mismatch between the large model space for $h$, and the much simpler derived quantity, i.e.,~the root $x^*$, to be learned. Moreover, to our knowledge, existing emulators for $h$ have few tools to take advantage of the specific structure that arises in root-finding, first and foremost the fact that there is a unique $x^*$. In our running example of a GP emulator, it is very challenging to control the behavior of the level-set; see for example the ongoing efforts to build tractable monotone GP models \cite{riihimaki2010gaussian}. By explicitly targeting $X^*$ we seek the most direct path to developing effective rules for the sequential design~$x_{1:n}$.
	
	Secondly, we propose another class of sampling policies which do not make explicit use of a data acquisition function such as the strategies above, but rather use solely the state variable $f_{n}$ (and perhaps additional randomization) in order to select new samples. Namely, sampling locations  are \emph{quantiles} of $f_{n}$, i.e., $x_{n+1} := F_{n}^{-1}(q_{n})$,  where $F_{n}(\cdot)$ is the CDF of $f_{n}$ and $q_n \in (0,1)$ are the sampling quantiles, which can either be randomized or fixed.  For instance, $q_{n} \equiv 1/2 \ \forall n\ge 1$ corresponds to the classic PBA median-sampling strategy. At the other extreme, taking $q_{n}\sim \mathsf{Unif}(0,1)$ -- which closely resembles Thompson Sampling~\cite{russo2016information}; new locations are chosen according to the current likelihood of~$X^*$.
	
	As mentioned, our analysis has been motivated by the interest in applying PBA in the context of approximate dynamic programming (ADP), where the objective loss function is defined primarily in terms of $\|\hat{x}-X^*\|$ rather than in terms of $\|\hat{h} - h\|$ \cite{gramacy2013sequential}. Towards this end, we seek classification/root-finding frameworks in lieu of the standard value-function-approximation methods in ADP. Another motivation has been our prior experience in working with ``highly stochastic'' simulators, i.e.,~settings with low signal-to-noise ratio and non-Gaussian, non-constant simulation noise $\eps(x)$. Far from being a statistical anomaly, this context is ubiquitous in the stochastic simulation applications we envision. When the noise properties are ``non-standard'', effective learning of $h$ becomes much more challenging, which leads to a considerable modeling, and eventually computational, effort in building a robust $\hat{h}$. Much of this work is however likely to be misplaced for the root-finding objective which is asymptotically local and hence much less sensitive to the noise. Therefore, we again see the root-finding paradigm advantageous by being less reliant on limiting statistical assumptions. Simultaneously, the more obscured is $h$ by the simulation noise $\eps$, the more valuable is structural knowledge about $x^*$.
	
	To provide further context for our work, let us recapitulate our contributions relative to existing methods. On the one hand, compared to PBA, we work with \textit{unknown} and \textit{location-dependent} $p(x_{n})$, which indeed requires a complete re-imagining of the algorithm. Moreover, we focus on practical solutions that work well in non-asymptotic settings, i.e.,~for a limited budget of $T$  available oracle calls. In particular, we contrast our strategies with one of the proposals in~\cite{waeber2013probabilistic,frazier2016probabilistic} that TPO approach to learn $p(x_{n})$. As we show, while TPO enjoys nice theoretical properties and is a viable alternative in terms of its asymptotic behavior, it performs poorly for small $T$.
	
	On the other hand, compared to simulation optimization, we develop a root-finding procedure which is built around the notion of constructing an explicit posterior for the root, and hence primarily operates with the knowledge state rather than a surrogate for $h(\cdot)$. This allows us to obtain and monitor the (pseudo-) credible bands for $X^*$ which give sequential quantification of the learning performed by PBA. Finally, compared to level-set estimation, we develop a different type of Expected Improvement criteria which are linked to the posterior entropy of $X^*$ rather than a $h$-based statistic, grounding our method in a purely information-theoretic paradigm.
	
	\subsection{Generalized Probabilistic Bisection}
	A complete solution to the SRFP defined by \eqref{eq:pba_response} was provided by \cite{waeber2013bisection} under the key assumption that $p(x) \equiv p$ is a \textit{known} and $x$-independent constant. Namely, Waeber et al.~derived the equations for the posterior density
	\begin{align}\label{eq:fn}
	g_{n}(X^{*}):= p(X^{*} | y_{1:n-1},x_{1:n-1})
	\end{align}
	and then established that sampling at the posterior median, $x_{n+1}:= G^{-1}_{n}(0.5)$ for all $n = 1,2,\ldots$, is an optimal policy. More precisely, they proved that this sampling rule minimizes the expected Kullback-Leibler (KL) distance for its utility function, and most importantly achieves exponential convergence
	for the estimate $\hat{x}_{n} \equiv x_n$ towards  $x^*$, i.e.,~$|\hat{x}_{n} -x^{*}| = \mathcal{O}(e^{-\alpha n})$ with an explicit expression for $\alpha > 0$.
	This justifies its name, as the PBA manages to effectively reduce the interval containing $x^*$  by $\alpha\%$ at each stage. This result is truly impressive both given the noisy oracle replies and thanks to the simplicity of the sampling rule. Moreover, PBA exemplifies the Bayesian setup: $x_{n+1}$ is selected based on the information summarized by $g_n$, which also yields the point estimate~$\hat{x}_n$.
	
	Some partial results extending to the case where $p(x)$ is non-constant (but still known) were given in~\cite{waeber2011bayesian}. The crucial assumption that oracle properties, specifically $p(x)$ is known, is
	hard to justify in the context of unknown response $h(\cdot)$. For example, when the noise component is $\epsilon(x) \sim \mathsf{N}(0,\sigma^{2}(x))$, then $p(x) = \Phi \left( |h(x)|/\sigma(x) \right)$, where $\Phi(\cdot)$ is the standard Gaussian CDF, and therefore knowledge of $p(\cdot)$ is equivalent to knowing the signal-to-noise ratio --- a  rather unlikely proposition. At the same time, this assumption is critical to the performance: as we discuss below without further modifications the PBA might fail completely in the context of unknown $p(x)$. More sophisticated sampling strategies are needed to resolve this tension between exploitation and exploration.
	
	To generalize the ideas of PBA to the setting of \eqref{eq:pba_simulator} we introduce a  knowledge state,  $f_{n}$, that is recursively updated and used for acquiring new samples. The underlying philosophy is a  Bayesian formulation of SRFP, translating the task of learning the root into the language of ``beliefs'' encapsulated by $f_{n}$ and used to quantify (posterior) uncertainty about $X^*$. Intuitively, $f_n$ is a ``surrogate'' to the true posterior $g_n$ that is no longer attainable due to unknown $p(\cdot)$.
	
	A key ingredient of our approach is the use of \emph{replications}: repeatedly evaluating the oracle $K \in \mathbb{N}$ times at a fixed sampling location $x$. The latter procedure allows us to obtain a point estimate $\hat{p}(x)$ based on the  i.i.d. responses~$y_{1:K}(x):= (y_{1}(x),\ldots,y_{K}(x))$ observed at $x$, which is then used to update knowledge from $f_{n}$ to  $f_{n+1}^{(K)}$. Replicates decouple the problems of learning $X^{*}$ and of learning $p(\cdot)$; they also boost the signal-to-noise ratio which allows faster convergence at the macro-level. We remark that this strategy is a common one for dealing with heteroscedastic stochastic simulators, see the related ideas in Stochastic Kriging \cite{ankennman:nelson:staum:2010}. Our resulting G-PBA framework learns in \textit{parallel} $X^*$ and $p(\cdot)$ and is  summarized in Algorithm~\ref{alg:algo_generalized_PBA}.
	
	\IncMargin{1em}
	\begin{algorithm}[htb]
		\small
		\SetKwInOut{Input}{input}\SetKwInOut{Output}{output}
		\Input{Total query budget $T = N \times K$, where $K$ is the batch size\; Prior distribution $X^{*}
			\sim f_{0}$ on the root.}
		\For{$n \leftarrow 0,1, \ldots,N-1$}{
			Generate next sampling point $x_{n+1}$\;\nllabel{alg:sampling}
			Obtain the estimate $\hat{p}(x_{n+1})$ using the oracle responses $y_{1:K}(x_{n+1})$\;
			Update knowledge state to $f_{n+1}^{(K)} := \Psi_{K} \left(f_n^{(K)},x_{n+1},B_{K}(x_{n+1});\hat{p}(x_{n+1}) \right)$\;
			\nllabel{alg:updating}
		}
		\Return Root estimate $\hat{x}_{N}$; Knowledge state $f_{N}^{(K)}$.
		\caption{Generalized PBA.}\label{alg:algo_generalized_PBA}
	\end{algorithm}\DecMargin{1em}
	In order to implement Algorithm~\ref{alg:algo_generalized_PBA}, the G-PBA must specify:
	\begin{enumerate}[label=(\roman*),align=left]
		\item statistical procedure $\hat{p}(x_{n+1})$ for estimating $p(x_{n+1})$ at $x_{n+1}$.
		\item the mechanism to update knowledge states $\Psi_{K}: f_{n} \rightarrow f_{n+1}$;
		\item the rule $\eta$ for selecting $x_{n+1} =\eta(f_{n})$ given $f_{n}$;
	\end{enumerate}
	
	All three of the steps (i)-(iii) require novel analysis, and are a part of our contributions.
	
	\textbf{Statistical procedure for estimating $p(\cdot)$. } All three steps above require knowledge of $p(x)$, so proper inference of the latter is central to the G-PBA performance. Because the oracle is  ``democratic'': $p(x) \ge 0.5 \forall x$, there is a majority-vote property, whereby the estimate is based on the majoritatively observed sign of the oracle responses $Y_{1:K}(x)$. This introduces a fundamental bias which becomes especially significant when sampling close to $X^*$ ($|h(x)|$ is small and $p(x)\simeq 0.5$).
	
	In Section \ref{sec:knowledge_states} we investigate three types of $p$-estimators: frequentist based on majority proportion; Bayesian based on the posterior of $p$ given $Y_{1:K}$; and boosted which directly aggregate oracle responses to  construct a subsidiary signal whose \textit{specificity} is enhanced thanks to batching.
	
	\textbf{Knowledge Updating Procedure. } To update $f_n$ we then plug-in an estimated $\hat{p}(x)$ into a knowledge state transition function of the form  \begin{equation}\label{eq:updating_states}
	f_{n+1}^{(K)} := \Psi_{K} \left(f_n^{(K)},x_{n+1},B_{K}(x_{n+1});\hat{p}(x_{n+1}) \right), \quad  n=0,1,\ldots,N-1 \quad \mbox{ and} \ K \in \mathbb{N},
	\end{equation}
	where the sufficient statistic $B_K(x_{n+1})$ counts the number of \textit{positive} values among $Y_{1:K}(x_{n+1})$. The map \eqref{eq:updating_states} is the analogue of Bayesian updating when $p(x)$ is known . Note that the knowledge transition  $\Psi_{K}(\cdot)$ function is similarly batched, allowing us to make full use (while maintaining computational efficiency) of the sampled replicates. This aspect is fully addressed in Section \ref{sec:knowledge_states}.
	
	\textbf{Sampling Policies. }Third, to select the locations  $x_{n+1}$ at each $n = 0,\ldots,N-1$  we introduce several sequential sampling policies $\eta$. The first family of  \emph{Information Directed Sampling} uses an information gain function $I_{n,K}(x, f_n)$ to quantify the learning rate for $X^*$ if a new query batch is done at $x$. It is motivated by the optimality property of standard PBA in terms of maximizing the KL relative entropy between $f_n$ and $f_{n+1}$.
	
	The second family of \emph{Quantile Sampling} is motivated by the other aspect of PBA, namely of sampling at the {median} of the knowledge state. Letting $F_{n}$ be the CDF of $f_{n}$, we therefore propose to use the quantiles of $f_n$ for selecting the next $x_{n+1}$:
	\begin{equation}
	\label{eq:fn_policy}
	x_{n+1}^{Q} := F_{n}^{-1}(q_{n}), \qquad \mbox{where $q_{n} \in (0,1)$}.
	\end{equation}
	Another important computational adjustment that we entertain is an additional degree of randomization which serves to (a) alleviate the issue of error accumulation arising from uncertainty in estimating $p(x_n)$ and (b) enforce exploration of the state space in order to accelerate convergence to the true $X^{*}$. Our experiments demonstrate the value of such randomized sampling policies and can be viewed as analogues of similar stochastic searches in Bayesian optimization (such as Thompson sampling~\cite{russo2016information}). Full analysis of these designs is in Section \ref{sec:sampling}.
	
	Note that due to batching G-PBA will have two different time scales: macro-iterations $n=1,\ldots$ corresponding to the query locations $x_{1:N}$ and \textit{wall-clock time}, $T = K \times N$, which counts the total number of oracle queries and hence the overall computational expense. In  our G-PBA paradigm $K$ is fixed; an alternative is to make $K(x)$ adaptive, for example via a modified hypothesis testing procedure. This is the approach of the TPO-PBA in which the batch size $K(x)$ is an unbounded random variable that depends on $Z_{1}, Z_2, \ldots$.
	
	\textbf{Estimating the root $X^{*}$. }The final ingredient is the rule $\hat{x}_{n}$ to construct an estimate of the root based on $f_n$. In analogy to the classical PBA setting we utilize for the remainder of the article the posterior median (which we find is generally more robust than say the mean, as $f_n$ is often skewed or multi-modal),
	\begin{equation}
	\label{eq:median_fn}
	\hat{x}_{n} := \mbox{median}(f_{n}).
	\end{equation}
	The rest of the paper is organized as follows. 
	
	Section \ref{sec:InformationGainCriterion} combines the methods developed for sampling selection and updating knowledge states to obtain a collection of Generalized PBA algorithms. In particular, we mix-and-match the three key  ingredients of our G-PBA paradigm: sampling strategy $\eta$, statistical procedure for learning $p(\cdot)$ and batch size $K$, denoted abstractly by the triple $(\eta,\hat{p},K)$, to examine their performance. Our results demonstrate that the resulting root estimates $\hat{x}_{n}^{(K,\eta)}$ have indeed low average absolute error as well as minimal average uncertainty. Towards this end, in Section \ref{sec:Synthetic_Example} we numerically benchmark using several test cases: first with three synthetic examples, followed by a real-life example motivated by the optimal stopping problem presented in Section \ref{sec:osp}. Finally, Section~\ref{sec:conclusion} outlines the main results and contributions about the proposed heuristics for SRFPs and states some of the possible applications of the latter methods.
	
	\section{Knowledge States}
	\label{sec:knowledge_states}
	
	Consider a real-valued continuous response function $h:~[0,1]~\rightarrow~\mathbb{R}$. For concreteness we have re-scale the (compact) input space to the unit interval. The function $h$ is noisily sampled via the stochastic simulator~\eqref{eq:pba_simulator}. Let $X^{*}$ be the random root location and $x^{*}$ its realized value at which $h(x^{*}) = 0$. To learn  $x^{*}$, the PBA works with the signs $Y(x) := \sign Z(x)$, which  due to the stochastic nature of the responses $Z(x)$, are \textit{correct} with probability
	\begin{equation}
	\label{eq:ProbCorrect}
	p(x;x^{*}) = \mathbb{P}\left(Y(x) = \sign( x^{*} - x)\right)
	\end{equation}
	
	Assuming that \eqref{eq:ProbCorrect} is known, the next Lemma provides the analytical one-step updating equations for the posterior $g_n$ of $X^*$ defined  in \eqref{eq:fn}.
	\begin{lemma}
		\cite{waeber2013bisection}
		\label{lemma:pba}
		Let $G_{n}(\cdot)$ be the CDF of $g_{n}$ and $p(x)$ as in \eqref{eq:ProbCorrect}. Define
		\begin{equation}
		\gamma_{n}(x;p(x)) := \mathbb{P}\left( Y_{n}(x)=+1 |x_{1:n-1},y_{1:n-1} \right) = p(x)[1-G_{n}(x)] + [1-p(x)]G_{n}(x),
		\label{eq:gamma.x}
		\end{equation}
	
		Given a prior $g_0(\cdot)$ we havethe recursion for $n=0,1,\ldots$
		\begin{subequations}\label{eq:pba}\begin{align}
			\mbox{If} \ \  Y_{n}(x) = +1 \  \mbox{then} \ \  g_{n+1}(u) &=
			\frac{1}{\gamma_{n}(x;p(x))} \left[ p(x_{n}) 1_{\{ u\geq x \} } + (1-p(x)) 1_{\{ u < x \}} \right]g_{n}(u); \\
			\mbox{If} \ \  Y_{n}(x) = -1 \  \mbox{then} \ \  g_{n+1}(u) &= \frac{1}{1-\gamma_{n}(x;p(x))} \left[ (1-p(x)) 1_{\{u\geq x\}}+ p(x) 1_{\{u < x\}} \right]g_{n}(u).
			\end{align}\end{subequations}
	\end{lemma}
	\begin{remark}
		If no prior knowledge about the root location $X^{*}$ is provided, then a sensible choice  is a vague prior $g_{0} = \mathsf{Unif}(0,1)$. The latter is also computationally convenient, since \eqref{eq:pba} then implies that $g_n$ will be piecewise constant $\forall n$, with discontinuities precisely at the sampled $x_{1:n}$. Therefore, storage and updating of $g_n$ becomes an $\mathcal{O}(n)$ operation in this setup.
	\end{remark}

	\subsection{Batched Querying}
	Abstractly, the updating \eqref{eq:pba} constitutes a \textit{knowledge transition function} $\Psi: g_n \mapsto g_{n+1},$ which takes as inputs the current knowledge state $g_{n}$,  the oracle response $Y_{n}(x)$ and its probability of correctness $p(x)$ when queried at the point $x\in (0,1)$. To learn $p(x)$, we employ batched queries, keeping the sampling location $x$ unchanged for $K \ge 2$ steps. Considering the resulting i.i.d.~sequence of oracle responses $Y_{1:K}(x)$, the knowledge state $g_{n}$  can be recursively computed by using the update~\eqref{eq:pba} $K$-times to obtain $
	g_{n+1}^{(K)}\equiv g_{n+K}$. Because $p(x)$ is the same across those updates, we can simply consider the total number of \textit{positive} oracle responses observed at $x$,
	\begin{equation}
	\label{eq:bk_postive_signs}
	B_{K}(x) := \textstyle \sum_{j=1}^{K}1_{\{Y_{j}(x) = +1\}},
	\end{equation}
	yielding an aggregated knowledge transition function from $g_{n}$ to $g_{n+1}^{(K)}$.
	
	\begin{theorem}
		\label{thm:batched_sampling}
		The \textit{batched} Bayesian updating, $\Psi_{K}$, which maps $g_{n}$ to $g_{n+1}^{(K)}$ is given by
		\begin{align}
		\label{eq:batched_updating_pba}
		g_{n+1}^{(K)}(u) &:= \Psi_{K}(g_{n}(u),x,B_{K}(x);p(x)) \nonumber \\
		&=  \left\{
		\begin{array}{ll}
		c_{n}^{-1}(x)\left[p(x)^{ B_{K}(x)} (1-p(x))^{K-B_{K}(x)} \right] \cdot g_{n}(u) & \mbox{if $0<x<u<1$,}\\
	 	& \\
		c_{n}^{-1}(x) \left[(1-p(x))^{ B_{K}(x)} p(x)^{K-B_{K}(x)}\right] \cdot 	 g_{n}(u) & \mbox{if $0<u\leq x<1$;}
		\end{array}
		\right.
		\end{align}
		for all $x\in(0,1)$ with normalizing constant $c_n(x) \equiv c(x,g_n(x),B_{K}(x),p(x))$
		\begin{equation}
		\label{eq:constant_batched}
		c_n(x):=\left[ (1-p(x))^{B_{K}(x)} p(x)^{K-B_{K}(x)} \right]G_{n}(x) +  \left[ p(x)^{B_{K}(x)} (1-p(x))^{K-B_{K}(x)} \right](1-G_{n}(x)).
		\end{equation}
	
	\end{theorem}	
	If we furthermore define the \textit{right scaling-factor}
	\begin{equation}
	\label{eq:rho}
	\rho(x,B_{K}(x);p(x)) := p(x)^{B_{K}(x)} (1-p(x))^{K-B_{K}(x)},
	\end{equation}
	then  the ratio
	\begin{equation}
	\label{eq:RightScalingFactor}
	R^{(K)}(g_{n},x,B_{K}(x);p(x)):= \rho(x,B_{K}(x);p(x))/c_{n}(x)
	\end{equation}
	completely specifies $\Psi_K$ in~\eqref{eq:batched_updating_pba}: given the total number of positive responses $B_{K}(x)$, the new posterior $g_{n+1}^{(K)}(u)$ is recovered by \textit{scaling} the values of $g_{n}(u)$ for $x\leq u$ by the factor $\rho$ from~\eqref{eq:rho} divided by the normalizing constant $c_n(x)$ from~\eqref{eq:constant_batched}. Hence, if  $B_{K}(x) > \left \lfloor{K/2} \right \rfloor$, i.e.,~there is favorable evidence that $x^{*}$ is rightwards of $x$, then the mass  of $g_{n+1}^{(K)}$ is shifted to the right of $x$. In the case where $p(x) \in \{0,1\}$,~\eqref{eq:rho} is defined by $\rho(x,B_{K}(x);p(x)):= p(x)$, which effectively reduces the support of $g_{n+1}^{(K)}$ by placing zero mass on one of the intervals that have $x$ as an end-point.
	
	For our G-PBA algorithms, neither \eqref{eq:pba} nor \eqref{eq:batched_updating_pba} are feasible, since they require the unknown $p(x)$. Nevertheless, to mimic the Bayesian updating paradigm we introduce an \emph{approximate} knowledge state $f_n$ which follows the transition function in \eqref{eq:batched_updating_pba} by plugging-in an appropriate estimate $\hat{p}_K(x)$
	\begin{equation}
	\label{eq:updating_batched}
	f_{n+1}^{(K)} := \Psi_{K}(f_{n},x,B_{K}(x);\hat{p}(x)); \ n=1,\ldots,N-1,\quad K\ge 2 \mbox{ and $x \in (0,1)$},
	\end{equation}
	where $\Psi_{K}$ is computed via Theorem~\ref{thm:batched_sampling}. Note that because \eqref{eq:updating_batched} is necessarily an approximation, $f_n$ does not match the true posterior $g_n$.
	
	\subsection{Frequentist and Bayesian Estimators for $p(\cdot)$}
	\label{sub:freqEstimators}
	
	The task in this section is to perform statistical inference on the unknown parameter $p(x)$ by using the batched i.i.d. responses $Y_{1:K}(x)$. The latter is done by re-parameterizing $p(x)$ via
	\begin{equation}
	\label{eq:ProbCorrectMax}
	p(x) 	= \max\{p^{+}(x),1-p^{+}(x)\}, \qquad \text{where} \quad p^{+}(x) := \mathbb{P}(Y(x) = +1)
	\end{equation}
	is the \textit{marginal} probability of observing a \textit{positive} sign at location $x$.
	For the remainder of the section we consider a single (macro)-iteration of the overall G-PBA, treating
	the sampling location $x$ as fixed and suppressed from the notation. To estimate $p$ we construct an estimator for $p^{+}$ and then plug into~\eqref{eq:ProbCorrectMax}.
	
	From a frequentist perspective, we recall that the Binomial proportion $B_{K}/K$ is an UMVUE for $p^+$ since $B_{K}\sim \mathsf{Bin}(K,p^{+})$. This yields the  \textit{majority proportion} estimator $\bar{p}$ obtained by replacing $p^{+}$ by $B_{K}/K$ in~\eqref{eq:ProbCorrectMax}:
	\begin{equation}
	\label{eq:emp_prop_estimator}
	\bar{p}\equiv \bar{p}(B_{K}) := \max\left\{B_{K}/K,1-B_{K}/K\right\}.
	\end{equation}
	
	In Appendix~\ref{prop:bias_p}, we show that $\E_p[\bar{p}] > p$ is necessarily biased high as soon as $p>1/2$. Intuitively, the bias  in \eqref{eq:emp_prop_estimator} is due to the possibility that the majority vote points in the wrong direction.
	
	An alternative is to assign a prior for $p$ and then construct a posterior  based on the likelihood provided by the batched responses $Y_{1:K}$. Using \eqref{eq:emp_prop_estimator}
	yields the respective conditional likelihood of $\bar{p}$ as  (see Appendix \ref{lemma:likelihood_maj_prop}): 
	\begin{align}\label{eq:P-likelihood}
	\mathbb{P}_{p}(\bar{p}(B_{K})= j/K) = \left\{
	\begin{array}{ll}
	\mathsf{Bin}(j;K,p) + \mathsf{Bin}(j;K,1-p), & \mbox{$j=0,1,\ldots, (\ceil{K/2}-1)$;}\\
	 & \\
	\mathsf{Bin}(K/2;K,p), & \mbox{$j=\ceil{K/2}$;}
	\end{array}
	\right.
	\end{align}
	where $\ceil{a}$ is the ceiling function, and $
	\mathsf{Bin}(j;K,p^{+})$ is the pmf of $\mathsf{Bin}(K,p^{+})$ evaluated at $j = 0,\ldots,K$. Assuming a vague prior $p \sim \mathsf{Unif}(1/2,1)$
	we then obtain explicitly the
	posterior density $\pi(p|\bar{p})$.
	
	\begin{theorem}
		\label{thm:posterior_p_pk}
		Suppose that $p$ has prior density $\pi_0(p) = 2\cdot 1_{\{p\in [1/2,1]\}}$. Then, for $K\geq 2$, the posterior density of $p$ conditioning on the majority proportion $\bar{p}(B_{K})=j/K$ is given by
		\begin{equation}
		\label{eq:Posterior_pdf_p}
		\pi(p| j/K) \propto \left\{
		\begin{array}{ll}
		p^{j}(1-p)^{K-j} +(1-p)^{j}p^{K-j}, & \mbox{if $j=0,1,\ldots, (\ceil{K/2}-1)$;}\\
		& \\
		p^{K/2}(1-p)^{K/2}, & \mbox{if $j=\ceil{K/2}$.}
		\end{array}
		\right.
		\end{equation}
	\end{theorem}
	Proof of Theorem \ref{thm:posterior_p_pk} is included in Appendix~\ref{sub:ProofsAndResults}.
	
	\begin{remark}
		Other priors (e.g.~location-dependent) for $p$ can be entertained. The $\mathsf{Uniform}$ choice is convenient both as a vague prior, and due to it matching the conjugate Beta-Binomial updates.
	\end{remark}
	
	Figure~\ref{fig:updating_comparison} shows the theoretical {expected} posterior density,
	$\hat{\pi}(p;x):= \mathbb{E}_{p^{+}}^{B_{K}}[\pi(p|\bar{p}(B_{K}))]$, obtained after averaging the posterior~\eqref{eq:Posterior_pdf_p} with respect to $B_{K}(x)\sim\mathsf{Bin}(K,p^{+}(x))$  for batch size values $K\in \{50,100,250,500\}$ and locations $x>x^{*}$ so that  $p(x) \in \{0.5,0.55,0.60,0.70\}$, using the test function~\eqref{eq:g-ex1} presented in Section~\ref{sec:Synthetic_Example} where $x^{*}=1/3$. It namely shows that posterior  is {unimodal} around the true $p(x)$; furthermore the posterior predictably tightens as $K$ increases locating most of the posterior mass around the true $p(x)$-value.
	
	\begin{figure}[ht]
		{
			\centering
			\includegraphics[width=0.90\textwidth]{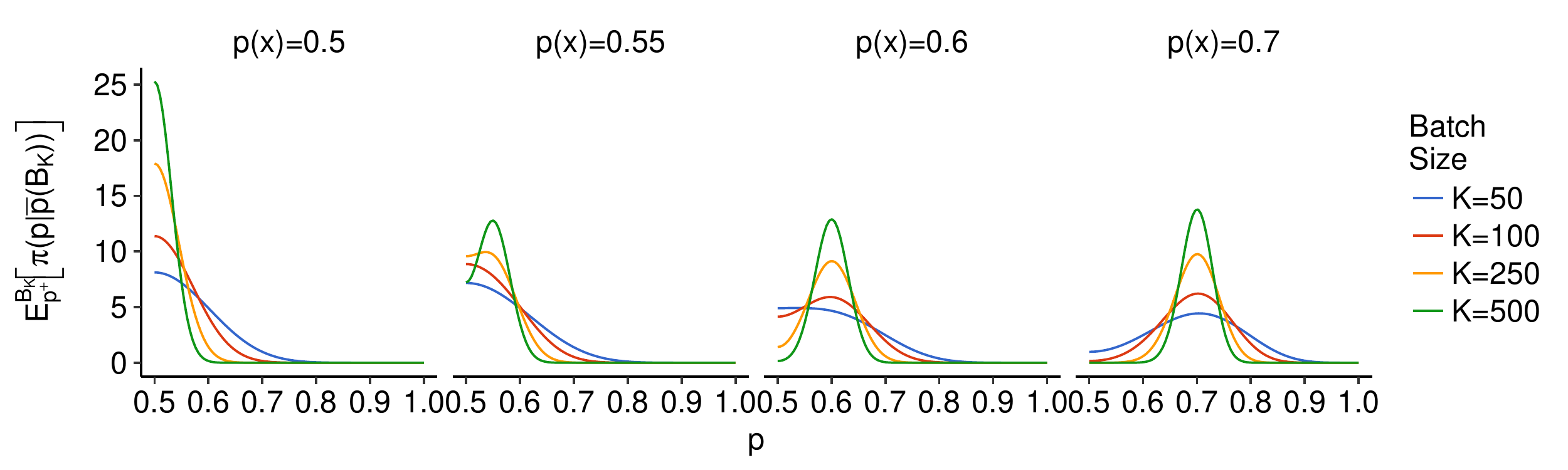}
			\caption{Expected posterior pdf $\hat{\pi}(p;x)$ obtained with respect to $B_{K}(x) \sim \mathsf{Bin}(K,p^{+}(x))$ for locations $x$ so that $p(x) \in \{0.50,0.55,0.60,0.70\}$ (columns) and batch size $K \in \{50,100,250,500\}$ (lines). \label{fig:updating_comparison}}
		}
	\end{figure}
	
	With $\pi(\cdot|\bar{p}(B_{K}))$ in closed-form, we can obtain a variety of estimators $\hat{p}_{\mathscr{L}}(\bar{p})$ by minimizing the Bayesian {posterior expected loss} for a given \textit{loss function} $\mathscr{L}(p,\hat{p})$. Namely,
	\begin{enumerate}[label = (\roman*)]
		\item \textit{posterior mode} based on $\mathscr{L}_{0}(p,\hat{p}) := 1_{\{|\hat{p} - p|>\epsilon, \epsilon>0\}}$ (taking $\epsilon \downarrow 0$ as $\pi(p|\cdot)$ is unimodal),
		\begin{equation}
		\label{eq:posterior_mode}
		\hat{p}_{\mathscr{L}_{0}}(\bar{p}) = \mbox{mode}\ \pi(p|\bar{p});
		\end{equation}
		\item \textit{posterior median} based on the ${L}_1$ loss $\mathscr{L}_{1}(p,\hat{p}) := |p - \hat{p}|$:
		\begin{equation}
		\label{eq:posterior_median}
		\hat{p}_{\mathscr{L}_{1}}(\bar{p}) = \mbox{median}\
		\pi(p|\bar{p}),
		\end{equation}
		\item and \textit{posterior mean} based on the $L_2$ loss $\mathscr{L}_{2}(p,\hat{p}) := (p - \hat{p})^2$:
		\begin{equation}
		\label{eq:posterior_mean}
		\hat{p}_{\mathscr{L}_{2}}(\bar{p}) = \mbox{mean}\ \pi(p|\bar{p})
		\end{equation}
	\end{enumerate}
	
	\begin{remark}
		Practically, \eqref{eq:posterior_mode} and \eqref{eq:posterior_median} have to be computed numerically, whereas \eqref{eq:posterior_mean} is computed in closed form as stated in Appendix~\ref{cor:posterior_mean}.
	\end{remark}
	
	The above Bayes estimators depend on \textit{four} different parameters: the sampling location $x$, realized number of positive responses at $x$ summarized via the majority proportion $\bar{p}(B_{K}(x))$; the batch size $K$ and the loss function $\mathscr{L}$. Whenever necessary we denote such dependency explicitly by $\hat{p}_{\mathscr{L}}(\bar{p}(B_{K}(x)))$.

	\begin{figure}[ht]
		{
			\centering
			{\includegraphics[width=0.75\textwidth]{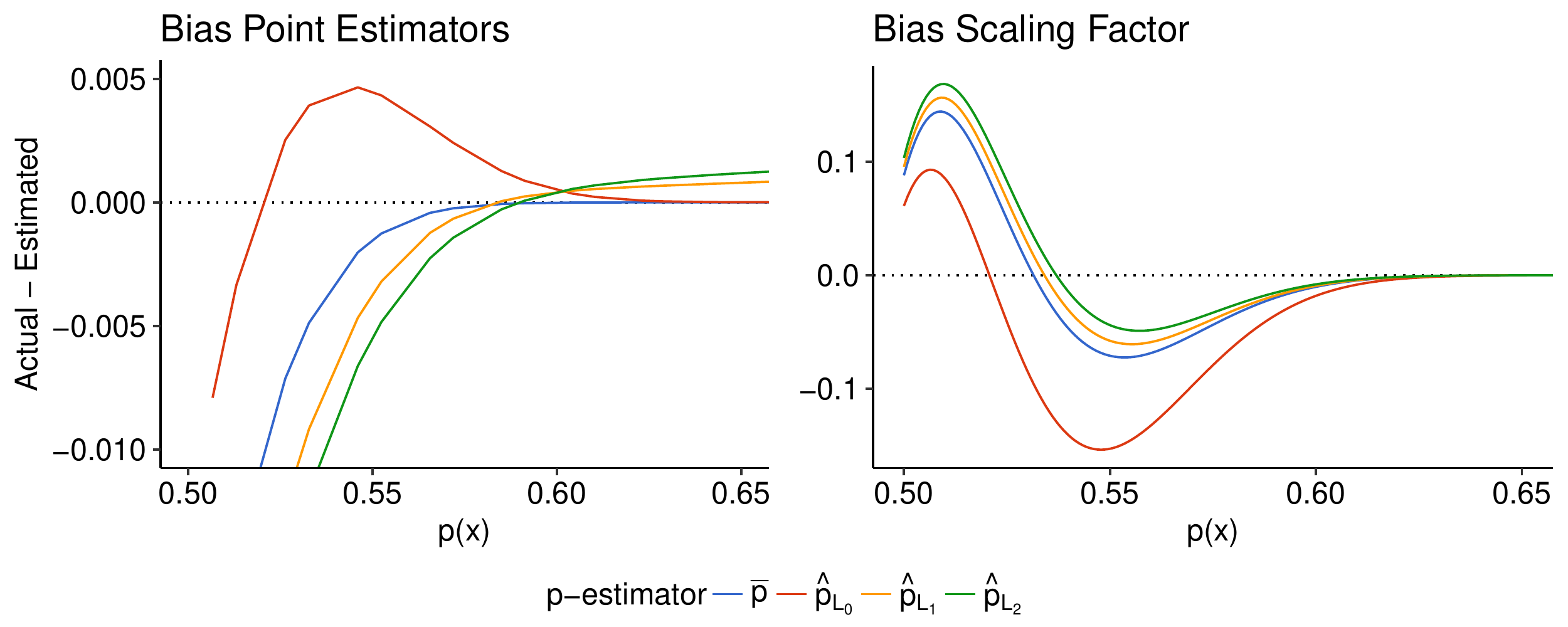}
				\caption{\emph{Left:} Expected bias of $\hat{p}$-estimators with respect to the number of positive responses $B_{K} \sim \mathsf{Bin}(K,p^{+})$. \emph{Right:} Expected right scaling factor $\hat{R}^{(K)}(f_{0},x,\hat{p})$ computed given $f_{0}\sim \mathsf{Unif}(0,1)$ and several locations $x>x^{*}$ so that $p(x)\in (0.50,0.70)$ ($x$-axis). Both panels are for $K=250$.
					\label{fig:biasPointEstimators}}
		}}
	\end{figure}

	The left panel of Figure~\ref{fig:biasPointEstimators} shows the theoretical expected bias $Bias_{p}(\hat{p}(x)) := \mathbb{E}_{p^{+}}^{B_{K}}[p(x) - \hat{p}(B_{K}(x))]$ corresponding to the estimators \eqref{eq:emp_prop_estimator}, \eqref{eq:posterior_mode}, \eqref{eq:posterior_median} and \eqref{eq:posterior_mean}; for $K=250$ and $p \in (0.5,1)$. Note that as $p \downarrow 0.5$, all procedures overestimate the true $p$, highlighting the difficulty to estimate $p(x)$ when $x\simeq x^*$. Of course, this issue is mitigated as batch size $K$ increases. The procedures which best approximate $p$ when $p\simeq 1/2$ are the \textit{posterior mode}, $\hat{p}_{\mathscr{L}_{0}}$, and the empirical majority proportion $\bar{p}$. However, as the true $p$ increases, $\hat{p}_{\mathscr{L}_{0}}$ underestimates $p$ (the bias increases), whereas the bias in the empirical majority proportion decays uniformly. Conversely, both the posterior mean $\hat{p}_{\mathscr{L}_{2}}$ and median $\hat{p}_{\mathscr{L}_{1}}$ overestimate when $p(x)\downarrow 1/2$ and underestimate it when $p(x)\uparrow 1$.

	\subsection{Bias in Knowledge States}\label{sub:biasAnalysisPBA}
	
	Recall that the key component about the update $f_{n+1}^{(K)}$ is given by the \textit{right-scaling factor}~\eqref{eq:RightScalingFactor} since it condenses all information needed in order to recover $f_{n+1}^{(K)}$ given $f_{n}$. The average scaling factor integrated against the distribution of $B_K$ is $\hat{R}^{(K)}(f_{0},x;\hat{p}) := \mathbb{E}_{p^{+}}^{B_{K}}[R^{(K)}(f_{0},x;B_{K},\hat{p}(B_{K}))]$, where $B_{K}\sim \mathsf{Bin}(K,p^{+})$. The right panel of Figure~\ref{fig:biasPointEstimators} shows the expected right-scaling factor obtained given a Uniform prior $f_{0}$ over $(0,1)$ and updating locations $x_{1}>x^{*}$ labeled via their $p(x_1)$ (x-axis). Since $x_{1}>x^{*}$, the right-scaling factor is expected to be close to zero when $p(x_1) \gg 0.5$ (since the updated $f_{1}$ would have fewer mass to the right of $x_{1}$) and conversely $\hat{R}^{(K)}(f_{0},x_1) \uparrow 1$ as $p(x_1) \downarrow 0.5$ (i.e.,~$x_{1}$ approaches the root $x^{*}$). We observe that in the latter setting, all four statistical procedures for  $\hat{p}$ tend to \textit{overestimate} the true right-scaling factor (the expected difference $R-\hat{R}$ is negative), meaning that there is ``overconfidence'' that $x^{*}$ is located to the right of $x_{1}$ even though in fact $p(x_{1})\cong 1/2$. In particular, the two statistical  procedures which seem to best resemble the true right-scaling factor when $x_{1}\simeq x^{*}$  are the \textit{ posterior mode}, as well as the \textit{empirical majority proportion}. Conversely, when the updating location $x_{1}$ is such that $p(x_{1})>1/2$, we see that all procedures provide an accurate description of the updated knowledge state at time $n=1$, especially for large values of $K$.
	
	The approximated posterior $f_{n+1}^{(K)}$ differs relative to the true posterior $g_{n+1}^{(K)}$
	due to the fact that $f_{n}$ utilizes the estimated $\hat{p}(x_{n+1})$ whereas $g_{n}$ uses the true  $p(x_{n+1})$. Figure~\ref{fig:PosteriorDensity} illustrates how the  bias in $\bar{p}$ induces over/under confidence when comparing the knowledge state $f_N$ vis-a-vis the ground truth $g_N$.
	Starting with a $f_0, g_0 \sim \mathsf{Unif}(0,1)$ prior, we compare the true posterior $g_{n}^{(K)}$ and  its approximation $f_{n}^{(K)}$ for $n\in \{1,2,3\}$ and $K =10$,  updated using the (arbitrary) locations $x_{1}=0.5$, $x_{2} = 0.4$ and $x_{3} = 0.2$, and the empirical proportion estimator $\bar{p}$ for the synthetic example~\eqref{eq:g-ex1}. Note that  the first two sampling locations $x_{1:2}$ are to the right of $x^{*}=1/3,$ whereas $x_{3}$ is leftwards of $x^{*}$.
	
	\begin{figure}[ht]
		{
			\centering
			\includegraphics[width=0.9\textwidth]{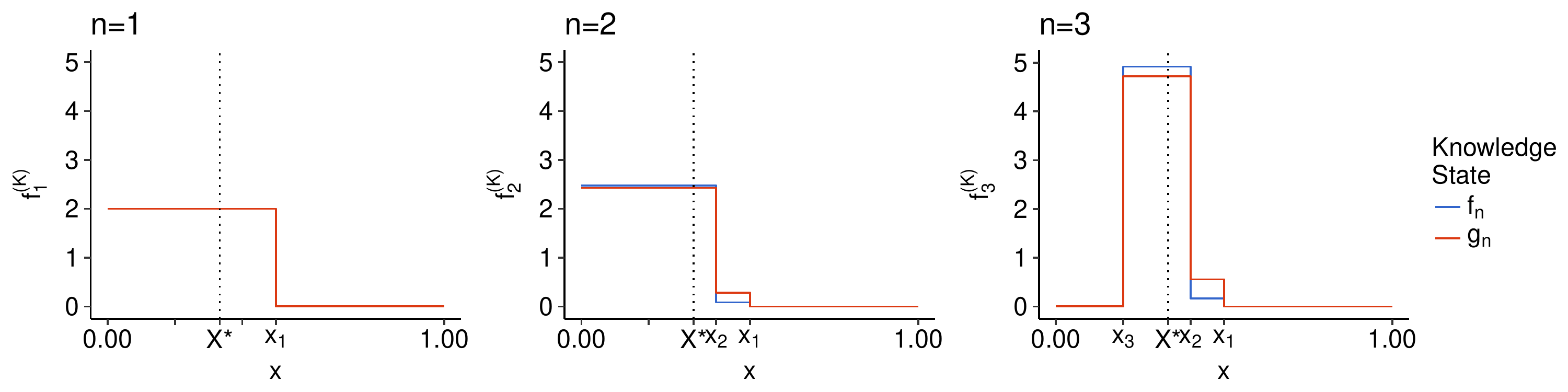}
			\caption{True and approximated knowledge states with the empirical proportion estimator $\bar{p}$ using three sampling locations $x_{1:3} = (0.5,0.4,0.2)$  and $K = 10$, using the linear function~\eqref{eq:g-ex1} with $x^{*}=1/3$.
				\label{fig:PosteriorDensity}}
		}
	\end{figure}
	
	\subsection{Aggregation of responses}
	
	An alternative strategy for updating the knowledge state is to build a subsidiary statistic from the i.i.d.~$Y_{1:K}$'s,  whose specificity is \textit{boosted} thanks to the batching. In other words, instead of using the $K$-step update $\Psi_K$ with $p$, we utilize a 1-step update $\Psi_1$ with an adjusted probability of correct response $\mathscr{P}$. Consider \textit{majority-vote} statistic $\mathscr{M}_{K}(x):=1_{\{B_{K}(x) > \ceil{K/2}\}}$~\cite{lam1994theoretical}. Then
	
	\begin{equation}
	\label{eq:px_majority_2}
	\mathscr{P}_{\mathscr{M}_{K}}(p) := \mathbb{P}_{p}( \mathscr{M}_{K}(x)  = 1_{\{x > x^*\}} ) = \textstyle \sum_{j=\ceil{K/2}}^{K}{K\choose j} p^{j}(1-p)^{K-j}.
	\end{equation}
	Substituting an estimate $\hat{p}$ in~\eqref{eq:px_majority_2} then yields $\mathscr{P}_{\mathscr{M}_{K}}(\hat{p}) = \textstyle \sum_{j=\ceil{K/2}}^{K} {K\choose j}   \hat{p}^{j}(1-\hat{p})^{K-j},$
	
	and the boosted update rule
	\begin{equation}
	\label{eq:updating_maj}
	f_{n+1}^{(K)} = \Psi_{1} \bigl(f_{n},x_{n+1},\mathscr{M}_{K}(x_{n+1});\mathscr{P}_{\mathscr{M}_{K}}(\hat{p}(x_{n+1}))\bigr).
	\end{equation}
	Note that since $\mathscr{M}$ only uses limited information about $B_K$, it is not sufficient for learning $p$. Consequently, the resulting knowledge state is not directly comparable to the full Bayesian posterior $g_n$; the hope is that through majority boosting we filter ``noise'' in $B_K$ and hence mitigate the bias in $\hat{p}$.

	\textbf{Aggregation of Functional Responses.} Assuming that the functional responses~\eqref{eq:pba_simulator} are available, another possibility for updating the knowledge state $f_{n}$ is to use the actual functional values $Z_{1:K}(x_{n+1})$ via the signal
	
	\begin{equation}
	\label{eq:consensus_signal_CLT}
	\mathscr{S}_{K}(x) := 1_{\{ \sum_{j=1}^{K}Z_{j}(x)>0\}}.
	\end{equation}
	By the CLT $\mathscr{P}_{\mathscr{S}_{K}}(h(x),\sigma(x)) :=   \mathbb{P}_{h,\sigma}(\mathscr{S}_{K}(x) =1_{\{x< x^{*}\}}) \simeq \Phi(\sqrt{K}|h(x)|/\sigma(x))$, where $\sigma^{2}(x)$ is the location-dependent
	variance of $\epsilon(x)$.  Observe that $\mathscr{P}_{\mathscr{S}_{K}}(h(x),\sigma(x))$ no longer depends on $p(x)$ but on the signal-to-noise ratio $h(x)/\sigma(x)$. A natural estimator is then
	\begin{equation}
	\label{eq:Prob_correct_clt}
	 \mathscr{P}_{\mathscr{S}_{K}}(\hat{h}_{K},\hat{\sigma}_{K})  =
	\Phi(\sqrt{K}|\hat{h}_{K}(x)|/\hat{\sigma}_{K}(x)); \quad \textstyle \hat{h}_{K}:= \frac{1}{K} \sum_{j=1}^{K}Z_{j} \  \mbox{and} \  \hat{\sigma}_{K}^{2}:= \frac{1}{K-1}\sum_{j  = 1}^{K}(Z_{j}-\hat{h}_{K})^{2}.
	\end{equation}
	
	Using the functional responses, the updated $f_{n+1}^{(K)}$ is thus computed using $\mathscr{S}_{K}(\cdot)$ via
	\begin{equation}
	\label{eq:updating_clt}
	f_{n+1}^{(K)} = \Psi_{1}\left(f_{n},x_{n+1},\mathscr{S}_{K}(x_{n+1});\mathscr{P}_{\mathscr{S}_{K}}
	(\hat{h}_{K}(x_{n+1}),\hat{\sigma}_{K}(x_{n+1}))\right).
	\end{equation}
	
	\begin{table}[H]
			\footnotesize
		\centering
	\caption{Schemes for knowledge state updating $f_{n+1}^{(K)}$ based on query batches of $K$ at location $x$.}
		\begin{tabular}{c|c|c}
			Update Scheme & Sufficient Statistic & Parameters  \\
			\hline
			$p$-estimate \eqref{eq:updating_batched} using $\bar{p}_{K}$ or $\hat{p}_{\mathscr{L}}(\bar{p})$ & $B_{K} = \sum_{j=1}^{K}1_{\{Y_{j} = +1\}}$ & $p$   \\
			Majority Boosting \eqref{eq:updating_maj} with $\mathscr{P}_{\mathscr{M}_{K}}(\hat{p})$ & $\mathscr{M}_{K} = 1_{\{B_{K} > \ceil{K/2} \}}$ & $p$  \\
			Functional Aggregation \eqref{eq:updating_clt} with $\mathscr{P}_{\mathscr{S}_{K}}(\hat{h}_{K},\hat{\sigma}_{K})$ & $\mathscr{S}_{K} = 1_{\{ \sum_{j=1}^{K}Z_{j}>0\}}$ & $h/\sigma$   \\
			\hline
		\end{tabular}
		\label{tab:trans_function}
	\end{table}
	
	\textbf{TPO Strategy. } A different aggregation of $Z_j$'s relies on hypothesis testing, specifically \textit{statistical tests of power one} (TPO) \cite{siegmund1985sequential}. The key idea is to use an adaptive number of replicates $K_\alpha(x)$ so as to boost the probability of correct response to level $p_\alpha$, without explicitly estimating $p(x)$ \cite{waeber2013probabilistic}. Let $S_{K}(x): = \sum_{j = 1}^{K} Z_{j}(x) $ and
	
	\begin{equation}
	\label{eq:TPOTermination}
	K_{\alpha}(x) := \min\{k \in \mathbb{N} : |S_{k}(x)| \geq c_{k}(\alpha)\};
	\end{equation}
	where $(c_{k}(\alpha))_{k \in \mathbb{N}}$ is defined in terms of the distribution of $\epsilon(x)$ and the significance parameter $\alpha \in (0,1)$. $K_\alpha$ is the adaptive batch size and the resulting output is the aggregated signal which is viewed as a test statistic for inference about the positivity of the drift of the random walk $S_\cdot(x)$.  The construction of $c_\cdot(\alpha)$ guarantees that
	$\tilde{p}(x) = \mathbb{P}(\tilde{Z}(x) = \sign (x^* - x) ) \geq 1 - \alpha/2$. To obtain the curved boundary $c_\cdot(\alpha)$ requires knowledge of the distribution of $Z(x)$. For example, if $Z(x) \sim \mathsf{N}(h(x),\sigma^{2})$ then $c_{k}(\alpha) = \sigma((n+1)[\log(n+1)-2\log \alpha])^{1/2}$.
	
	Table~\ref{fig:TPO-analysis} shows the average hitting time $\E_p[{K}_{\alpha}(x)]$ as well as its standard deviation $\hat{\sigma}_{K_{\alpha}}$ (in parentheses) for different $p(x), \alpha$ combinations. It illustrates
	that the expected batch size grows exponentially as $p(x) \downarrow  1/2$, which might be counterproductive in cases where the sampling budget is small. Indeed, instead of trying other locations, TPO will stubbornly sample the same $x$ thousands of times.
	
	\begin{table}[ht]
			\footnotesize
		\centering
			\caption{Average hitting time $\E[{K}_{\alpha}(x)]$ and corresponding standard deviation (in parentheses) to learn $p(x)$ using the TPO rule \eqref{eq:TPOTermination} with $\alpha \in \{0.05,0.10,0.20,0.40\}$ for the $h_1$ function in \eqref{eq:g-ex1}. Results are based on 1,000 macro runs.
			\label{fig:TPO-analysis}
		}
		\begin{tabular}{r|r|r|r|r}
			$p(x)$ & $\E[{K}_{0.05}(x)]$ & $\E[{K}_{0.1}(x)]$ & $\E[{K}_{0.2}(x)]$ & $\E[{K}_{0.4}(x)]$ \\
			\hline
			0.52 & 4951 (3209) & 4352 (3151) & 3563 (2983) & 2715 (2843) \\
			0.55 & 692 (483) & 594 (457) & 456 (403) & 362 (388) \\
			0.60 & 159 (113) & 133 (103) & 105 (95) & 79 (81) \\
			0.70 & 34 (24) & 29 (20) & 23 (18) & 18 (17) \\
			\hline
		\end{tabular}
	\end{table}
	\section{Sampling Policies}\label{sec:sampling}
	Sampling is the process of selecting querying locations so that the knowledge about the root  $X^{*}$ can be improved. In the context of the SRFP, the challenge is that sampling close to the root yields uninformative oracle responses. More specifically, since $x \to x^*$ implies $p(x) \downarrow 1/2$, the knowledge obtained from sampling in a vicinity of  $x^{*}$ is minimal and the updated state $f_{n+1}$ will change very little with relative to $f_{n}$. To resolve this challenge we investigate two classes of sampling policies that enforce \textit{exploration}:
	
	\begin{enumerate}[label=(\roman*)]
		\item \textit{Information Directed Sampling (IDS).} Inputs are greedily selected by optimizing the expected information gain about $X^{*}$ and the response $Y(x)$;
		\item \textit{Posterior Quantile Sampling.} Inputs are selected based on  $f_{n}$-quantiles.
	\end{enumerate}
	Both  deterministic and randomized versions of each class are analyzed.
	
	\subsection{Information Directed Sampling}
	\label{sec:InformationGainCriterion}
	
	This sampling strategy is driven by the notion of an \textit{acquisition function} which quantifies expected information gain from a new oracle query. A common information-theoretic approach is to maximize the KL divergence between the current knowledge state $f_{n}$ and its expected update $f_{n+1}$ conditional on sampling at a given $x$. The relative entropy between $f_n$ and $f_{n+1}$ can be interpreted as the mutual information between oracle $Y(\cdot)$ and $X^{*}$.
	This idea is similar to entropy-maximizing EI strategies (see e.g.~\cite{HernandezHoffman14}) and leverages the explicit form of KL-divergence when $p(x)$ is known.
	
	\begin{lemma}
		\label{lemma.condmutual}
		\cite{jedynak2012twenty} The expected information gain with regard to the root location $X^{*}$ from sampling at $x$ given $g_{n}$ (and $p(x)$) is
		\begin{align}
		\label{eq:InformationCriterion}
		\mathcal{I}(x,g_{n};p(x)) &:= -\gamma_{n}(x;p(x))\log \gamma_{n}(x;p(x)) - [1-\gamma_{n}(x;p(x))]\log [1-\gamma_{n}(x;p(x))]\nonumber \\
		&+ p(x)\log p(x) + (1-p(x))\log (1-p(x)),
		\end{align}
		where $\gamma_{n}(x;p(x))$ and $p(x)$ are given by \eqref{eq:gamma.x} and \eqref{eq:ProbCorrect}, respectively.
	\end{lemma}
	
	A greedy IDS strategy then myopically maximizes the information gain, $\mathcal{I}(x,f_{n}; {p}(x))$. As shown in~\cite{waeber2013bisection}, this myopic sampling rule is in fact \emph{optimal} for the global problem of minimizing the KL distance between $f_N$ and $X^*$. This approach has also been adopted in~\cite{jedynak2012twenty} for similar problems appearing in computer vision or, more recently in~\cite{russo2014learning} for on-line optimization problems.
	
	To implement the IDS approach, two modifications are necessary. First,
	similar to Section~\ref{sec:knowledge_states}, given the majority response $\bar{p}(x)$
	we can obtain \emph{a posteriori} plug-in version of~\eqref{eq:InformationCriterion}:
	
	\begin{equation}
	\label{eq:est-InformationCriterion}
	I_{n,K}(x;\hat{p}(x)) := \mathcal{I}(x,f_{n};\hat{p}(x)).
	\end{equation}
	
	Second, the maximization over $x$ can only be done \emph{ad hoc}, since computing \eqref{eq:est-InformationCriterion} can only be applied after querying the oracle $K$-times at $x$. As a work-around, we carry out the optimization over a discrete candidate set $\mathcal{S}_{M}(f_{n}) := \tilde{x}_{1:M}^{(n)}$: one picks $M\ge 2$, candidate locations $\tilde{x}_{1:M}^{(n)}$ using $f_{n}$, queries the oracle $K$-times at each $\tilde{x}_{i}^{(n)}$ and finally updates $f_{n}$ at the maximizer of this criterion:
	\begin{equation}
	\label{eq:greedy_policy}
	x_{n+1} := \argmax_{\tilde{x}_{i} \in  \mathcal{S}_{M}(f_{n})} I_{n,K}(\tilde{x}_i; \hat{p}(\tilde{x}_{i} ) )
	\end{equation}
	
	To construct candidate sets $\mathcal{S}_{M}(f_{n})$ we rely on the quantiles of $f_n$:
	
	\textbf{Deterministic IDS:} The test locations $\tilde{x}_{1:M}^{(n)}$ are \textit{fixed} posterior quantiles of $f_{n}$, i.e.,
	\begin{equation}
	\label{eq:quantiles_ids}
	\tilde{x}_{i}^{(n)} := F_{n}^{-1}(q_{i}).
	\end{equation}
	
	\textbf{Randomized IDS:} The test locations are randomly chosen posterior quantiles of $f_{n}$:
	\begin{equation}
	\label{eq:random_quantiles_ids}
	\tilde{x}_{i}^{(n)} = F_{n}^{-1}(q_{i,n}), \qquad q_{i,n} \sim \mathsf{Unif}(0,1).
	\end{equation}
	
	Note that at each iteration $n=1,2,\ldots$, $K \times M$ queries are made ($K$ at each $\tilde{x}_i$), of which only $K$ are used for actual updating to $f_{n+1}$. Therefore, after $N$ updates used for $f_N$, total wall-clock time is $T = N\times K\times M$. To minimize this inefficiency in our experiments we use $M=2$, so that \eqref{eq:greedy_policy} is reduced to comparing information gain at two chosen locations $\tilde{x}_{1:2}^{(n)}$.
	
	To illustrate the acquisition function $\mathcal{I}$, Figure~\ref{fig:testFunctions} shows the theoretical information gain using the three test functions from Section~\ref{sec:Synthetic_Example} (columns) and a Uniform prior $f_{0}$. We observe that the form of $x \mapsto \mathcal{I}(x,f_{n},p(x))$ is typically bimodal, driven by the fact that $ \mathcal{I}(x^{*},f_{n},p(x^{*})) = 0$ since $p(x^{*}) =1/2$: the information gain of sampling exact at the root location $x_{1} = x^{*}$ is zero. There are two local maxima located  both sides of $x^{*}$, with the higher being  the one with larger oracle accuracy $p(\cdot)$.
	
	\begin{figure}[htb]
		{
			\centering
			\includegraphics[width=0.80\textwidth]{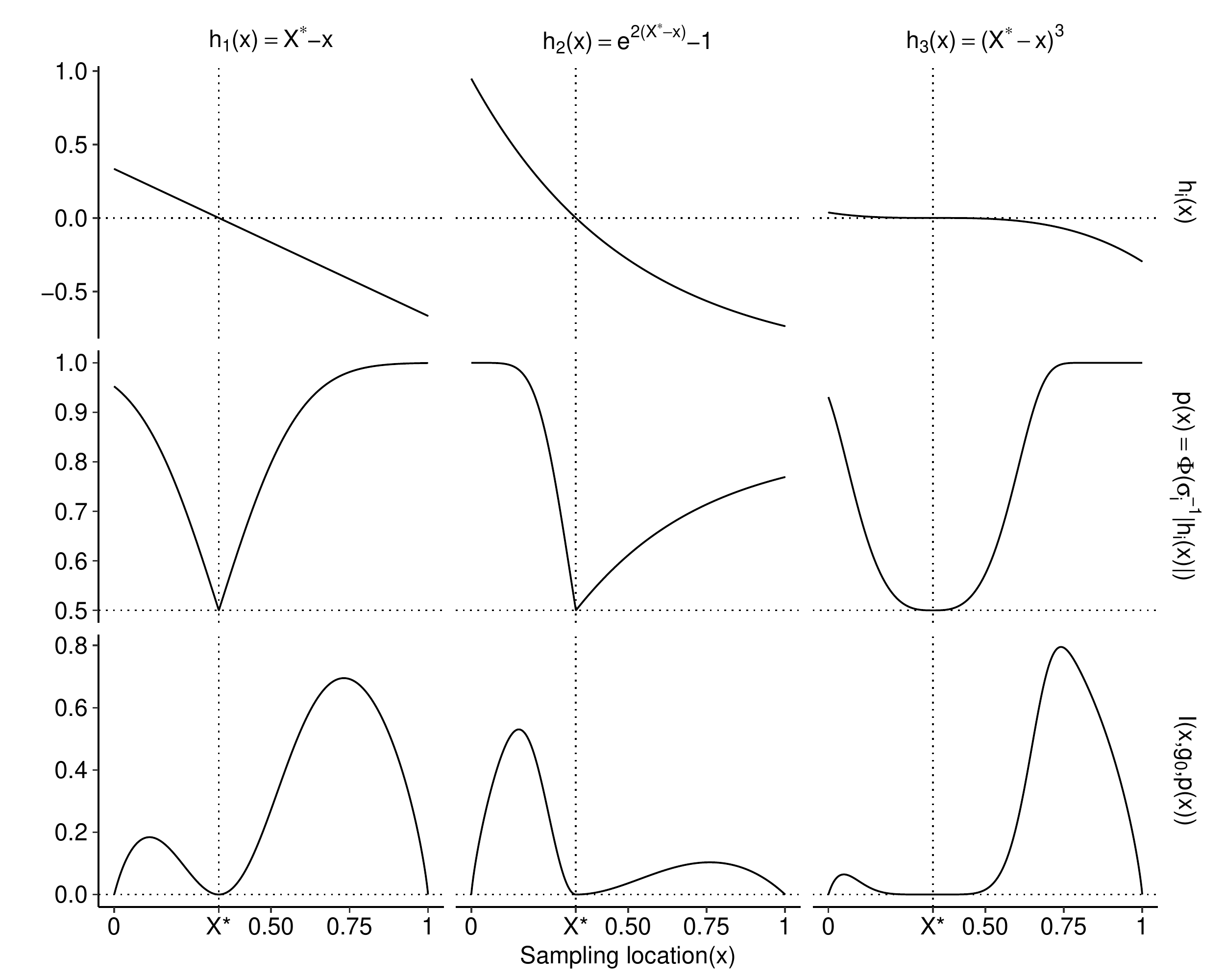}
			\caption{Synthetic test cases for Section \ref{sec:Synthetic_Example}. \emph{Top:}
	Test functions $h_{i}(x)$; \emph{Middle:} probability of correct response $p_{i}(x) = \Phi(\sigma_{i}^{-1}|h_{i}(x)|)$; \emph{Bottom:} information criterion $\mathcal{I}(x,f_{0},p(x))$ with $f_{0}\sim \mathsf{Unif}(0,1)$ prior. The vertical dashed line indicates the root $x^{*} = 1/3$, with $h_{i}(x^{*}) = 0$, $p_i(x^{*}) = 1/2$ and $\mathcal{I}(x^{*},f_{0},p_{i}(x^{*}))=0$.\label{fig:testFunctions}}
		}
	\end{figure}

	\subsection{Posterior Quantile Sampling}
	\label{sec:quantile_sampling}
	
	The message of classical PBA is that one should sample at the \textit{median} of the knowledge state $g_n$. However, this no longer holds when the $p(x)$ depends on the location $x$ since $p(x)\to 1/2$ as  $x \to x^*$. In fact, we show in our numeric examples that the performance of this policy does as bad as sampling uniformly over the input space in terms of uncertainty minimization. Intuitively, sampling at the median  is not suitable since after a few iterations the median is located too close to the root and therefore minimal information gain is obtained (this was already pointed out in \cite{waeber2013probabilistic}).
	
	Thus, other \textit{posterior quantiles} are explored, taking $x_{n+1} = F_n^{-1}( q_n)$. On the one hand, quantile sampling places samples where most of the posterior mass of  $f_{n}$ is located (which after a few iterations will be concentrated around $x^{*}$), allowing to gradually focus on the neighborhood of $x^*$. On the other hand,  quantile sampling is based solely on the knowledge state variable $f_{n}$ and can be used \emph{a priori} without yet having an estimate of $p(x_{n+1})$.

	\textbf{Systematic Quantile Sampling. }
	Locations are selected by \textit{systematically} iterating over $M\ge 2$ posterior quantiles $\check{q}_{0:M-1}$, fixed a priori. Then, in the $n$-th iteration, the next design point is
	\begin{equation}
	\label{eq:syst_sampling}
	x_{n+1} = F_{n}^{-1}(\check{q}_{(n\mod M)}).
	\end{equation}
	We remark that the precise ordering of  $\check{q}_n$'s will affect the results of Syst-Q. To balance the trade-off between exploration and exploitation we look at quantiles that are away from the median $q_n = 0.5$. Considering the shape of $\mathcal{I}$, a sensible rule is to consider the quartiles of $f_n$, i.e.,~$\check{q} \in \{0.25, 0.75\}$.
	
	\textbf{Randomized Quantile Sampling. }
	The next design point is a randomly chosen quantile of the posterior distribution $f_{n}$:
	\begin{equation}
	\label{eq:randomized_policy}
	x_{n+1} = F_{n}^{-1}(q), \ \mbox{where $q\sim \mathsf{Unif}(0,1)$.}
	\end{equation}
	The policy \eqref{eq:randomized_policy} can be interpreted as sampling at a location $X_n \sim f_n$, i.e.,~sampling based on the posterior distribution of $X^*$.
	
	\subsection{Batch size $K$}
	
	An essential tuning parameter in  Algorithm~\ref{alg:algo_generalized_PBA} is the batch size $K \ge 3$ needed to learn $p(x)$ at each updating location $x$. Recall that the total number of learning iterations is
	$N := \left \lfloor{T/K}\right \rfloor$. Therefore,  for a fixed budget $T$,  the batch size $K$ controls the balance between the learning of $X^{*}$ and $p(\cdot)$. When $K$ is small (thus $N$ large), the algorithm is exploring many sampling locations to learn $X^*$. When $K$ is large, the algorithm exploits the oracle in order to estimate $p(x)$ locally with high accuracy. As a result, for large values of $K$ the estimated $\hat{p}(x_{1:n})$ is likely to be close to $p(x_{1:n})$ and therefore  $f_{N}$ resembles the true posterior $g_{N}$. Consequently,  the probabilistic representation about $X^*$ would be excellent (measured, for instance, in terms of the $f_{n}$-coverage).
	However this would come  at the cost of sampling at very few sampling locations $x_{1:N}$, and the resulting limited knowledge about $X^{*}$ would lead to potentially larger residuals $|\hat{x}_{N} - X^{*}|$. In contrast, for $K$ small, the estimated $\hat{p}(x_{1:n})$ is highly biased and $f_{n}$ will significantly differ from the true posterior $g_{n}$ causing $f_{n}$ to collapse to regions where $X^{*}$ may not be located. As we show in our numerical examples, the latter case turns out to be more problematic. In particular we observe that moderately large $K \in [100, 500]$'s are necessary to obtain a reasonable $f_N$; otherwise the bias accumulates quickly.
	
	\section{Numeric Examples}
	\label{sec:Synthetic_Example}
	
	We proceed to present a series of numerical results in order to empirically assess the performance of the G-PBA algorithms introduced above. In analogy to \cite{waeber2013probabilistic}, we utilize the  following three test functions defined for $x \in (0,1)$, cf.~Figure~\ref{fig:testFunctions}:
	
	\begin{enumerate}
		\item The linear function
		\begin{align}
		h_{1}(x) &= x^{*} -x ,  \qquad\qquad \sigma_1(x) = 0.2; \label{eq:g-ex1}
		\shortintertext{\item the exponential function}
		h_{2}(x) &= \exp \{ 2(x^{*} - x)\} - 1, \qquad\quad \sigma_2(x) = 0.2 \cdot 1_{\{x < 1/3\}} + 1 \cdot 1_{\{x > 1/3\}};  \label{eq:g-ex2}
		\shortintertext{\item and the cubic function}
		h_{3}(x) &= (x^{*} - x)^{3}, \qquad\qquad \sigma_3(x) = 0.025. \label{eq:g-ex3}
		\end{align}
	\end{enumerate}
	In all cases we fix $x^{*}=1/3$. Example \eqref{eq:g-ex1} corresponds to a simple function where most of the stochastic root-finding procedures should work well, since its slope is constant and significantly different from zero in locations close to the root $X^{*}$. The curvature of \eqref{eq:g-ex2} creates an asymmetry in sampling: a measurement leftwards of $X^{*}$ yields a correct response with higher probability relative to  a measurement to the right of the root. Consequently, $f_n$ is expected to be skewed. Finally, example \eqref{eq:g-ex3} represents a difficult root-finding setting due to $h_{3}^{\prime}(X^{*}) = 0$, which implies that $p(x) \simeq 1/2$ for $x$ in the vicinity of $X^{*}$.
	We  moreover use a Uniform prior $X^{*} \sim f_{0} \equiv \mathsf{Unif}(0,1)$ and Gaussian noise $\epsilon(x) \sim \mathsf{N}(0,\sigma^{2}(x))$ as specified in \eqref{eq:g-ex1}-\eqref{eq:g-ex3}, so that the function evaluations $Z(x)$ in \eqref{eq:pba_simulator}  are normal random variables with mean $\E[Z(x)]  = h(x)$ and variance $\mbox{Var}(Z(x))=\sigma^{2}(x)$, with corresponding $p_i(x) = \Phi\left( |h_i(x)|/\sigma_i(x)\right)$ (as shown in the second row of Figure~\ref{fig:testFunctions}).
	
	To assess algorithm performance, we mix-and-match the three components that the user must pick: the sampling policy $\eta$, estimation method $\hat{p}$ for $p(\cdot)$, and batch size $K$ as follows:
	\begin{enumerate}

		\item \textit{Sampling policies $\eta$}:

		\begin{itemize}
			\item Deterministic IDS (Det-IDS): select the posterior quantile with highest information value using \eqref{eq:quantiles_ids}. Specifically use \eqref{eq:greedy_policy} with pre-fixed $\tilde{x}_{1:M}$.
			\item Randomized IDS (Rand-IDS): maximize the IDS criterion among $M$ random quantiles of $f_n$ as in~\eqref{eq:random_quantiles_ids}: use \eqref{eq:greedy_policy} with $\tilde{x}_{1:M} \sim \mathsf{Unif}(0,1)$;
			\item Randomized Quantile Sampling (Rand-Q): $x_{n+1} \sim f_n$ as in \eqref{eq:randomized_policy}. 
			\item Systematic Quantile Sampling (Syst-Q): $M$ pre-specified quantiles of $f_n$ that are systematically rotated  using \eqref{eq:syst_sampling}.
		\end{itemize}
		
		\item $p$-\textit{estimators}:  (i) the empirical majority proportion $\bar{p}$ from~\eqref{eq:emp_prop_estimator}; (ii) the posterior mode $\hat{p}_{\mathscr{L}_{0}}$~\eqref{eq:posterior_mode}; (iii)  posterior median $\hat{p}_{\mathscr{L}_{1}}$~\eqref{eq:posterior_median}; (iv) posterior mean $\hat{p}_{\mathscr{L}_{2}}$~\eqref{eq:posterior_mean}; (v) boosted $\mathscr{P}_{\mathscr{M}_{K}}(\bar{p})$~\eqref{eq:px_majority_2} (combined with empirical proportion $\bar{p}$); (vi) aggregated functional responses $\mathscr{P}_{\mathscr{S}_{K}}(\hat{h}_{K},\hat{\sigma}_{K})$~\eqref{eq:Prob_correct_clt}.
		
		\item \textit{Batch size $K$}: several batch values $K$ are used to learn $p(\cdot)$. Namely, we use $K\in \{50,250,500\}$.
	\end{enumerate}
	
	Finally, we also compare to the TPO policy that takes $x_{n+1} = F_{n}^{-1}(1/2)$, and performs a random number of oracle calls $K_\alpha(x_{n+1})$ based on \eqref{eq:TPOTermination}. For the latter we plug-in the true oracle variance $\sigma^2(x)$ and truncate sampling if it does not terminate by final clock-time $T$: $\tilde{K}(x_{n}) := \min\{T-\sum_{j = 1}^{n-1}K_{\alpha}(x_{j}),K_{\alpha}(x_{n})\}$ with the resulting $Z$-based estimator $\mathscr{P}_{\mathscr{S}_{\tilde{K}}}(\hat{h}_{\tilde{K}}(x_{n}),\hat{\sigma}_{\tilde{K}}(x_{n}))$. We try two boosting levels $\alpha \in \{0.05, 0.4\}$.
	
	In summary, the scheme space $(\eta,\hat{p},k)$ consists of 4 sampling policies $\eta$, 6 estimation methods for $\hat{p}$  and 3 batch sizes $K$, plus two versions of the TPO procedure, for a total of $6\times 4 \times 3 +2= 74$ combinations.

	To make all policies comparable we fix the total oracle queries (aka \textit{wall-clock} iterations) $T$ and write $f_T$ for the resulting terminal knowledge state. Recall that $T = N \times M \times K$ where $M=1$ for the Quantile sampling policies and $M=2$ for the IDS-based sampling policies.
	Below we summarize the overall parameters that are used to deploy the above estimation/sampling schemes:
	
	\begin{itemize}
		\item We use the $M=2$ quartiles of $f_{n}$,  $\check{q}_{0:1}=(0.25,0.75)$ for the Det-IDS and Syst-Q policies.
		
		\item $T = 20,000$ wall-clock iterations are used with batch sizes $K \in\{50,250,500\}$, resulting in $N^{IDS}\in\{200,40,20\}$ for the IDS strategies and $N\in\{400,80,40\}$ design points for the Q-based policies, respectively.
		\item All  metrics are estimated using $1,000$ Monte-Carlo (MC) macro-replications.
	\end{itemize}

	To evaluate the quality of the knowledge state variable $f_{T}:=f_{N}^{(K\times M)}$ given a fixed configuration  $(\eta,\hat{p},K)$ we use  the following three performance metrics:
	
	\textbf{Absolute residuals}: to determine the accuracy of the estimator $\hat{x}_n = \text{median}(f_n)$ we consider the $L_1$ \textit{residuals}:
	\begin{equation}
	\label{eq:EvalMeasure-3}
	r_{K}^{\eta}(f_{n}) := |\hat{x}_{n,K}^{\eta} - x^{*}|.
	\end{equation}

	\textbf{Credible intervals}: since $f_n$ is a surrogate for the posterior of $X^*$, we can evaluate the degree of uncertainty associated to the unknown root location $X^{*}$ via its credible interval (CI)
	
	\begin{equation}
	\label{eq:EvalMeasure-1}
	l_{K,1-\alpha}^{\eta}(f_{n}):= F_{n}^{-1}(1-\alpha/2) - F_{n}^{-1}(\alpha/2),
	\end{equation}
	i.e.,~the length of a symmetric $(1-\alpha)$\% CI between the  $\alpha/2$ and $(1-\alpha/2)$ percentiles of $f_{n}$.
	
	\textbf{Coverage}: the coverage probability
	\begin{equation}
	\label{eq:EvalMeasure-2}
	c_{K,1-\alpha}^{\eta}(f_{n}):= \E \left[ 1_{\{x^{*} \in [F_{n}^{-1}(\alpha/2),F_{n}^{-1}(1-\alpha/2)]\}} \right]
	\end{equation}
	measures the quality of $f_n$ relative to the exact Bayesian posterior. If $c_{1-\alpha}(f_N) \ll (1-\alpha)$ the coverage test indicates that $f_N$ prematurely collapses or equivalently overstates its confidence about $X^*$.
	
	Namely, small CI length $l_{K,1-\alpha}$ relative to residuals $r$ will lead to low coverage $c$.
	For both the coverage and the length of the credible  interval we use $\alpha=0.05$.
	
	\subsection{Sensitivity Analysis}
	\label{sub:SensitivityAnalysis}
	
	We use the linear test function~\eqref{eq:g-ex1} as our running example to illustrate the empirical performance of the G-PBA. The results for the exponential and cubic functions~\eqref{eq:g-ex2} and~\eqref{eq:g-ex3} are presented in Appendix~\ref{sub:ExponentialCubicResults}.
	
		Figure~\ref{fig:PerformanceStatistics} compares the performance of different $(\eta, \hat{p}, K)$ schemes as a function of wall-clock budget $T$. In terms of estimating $p(\cdot)$, the best method is unsurprisingly the CLT approximation $\mathscr{P}_{\mathscr{S}_{K}}(\hat{h}_{K},\hat{\sigma}_{K})$ which directly leverages the functional responses $Z(\cdot)$.
	This quantifies the intuition that $Z$'s carry more information than the sign-based oracle responses in~\eqref{eq:pba_response}. As a consequence, using $\mathscr{P}_{\mathscr{S}_{K}}$ leads to lower residuals and better coverage, i.e.,~better recovery of the correct posterior distribution (due to smaller updating errors). Among $B_K$-based methods, two good choices are the posterior mode $\hat{p}_{\mathscr{L}_{0}}$ and the empirical proportion $\bar{p}$. Both of these maintain a good balance between uncertainty reduction and low absolute residuals. Recall that these procedures were shown to be conservative in over-estimating $p(x)$ and hence better in controlling the bias in the updating of $f_n$, cf.~Section~\ref{sub:biasAnalysisPBA}. This is important in the latter stages as $p(x_n) \simeq 1/2$.

	\begin{figure}[htb]
		{
			\centering
			\includegraphics[width=0.80\textwidth]{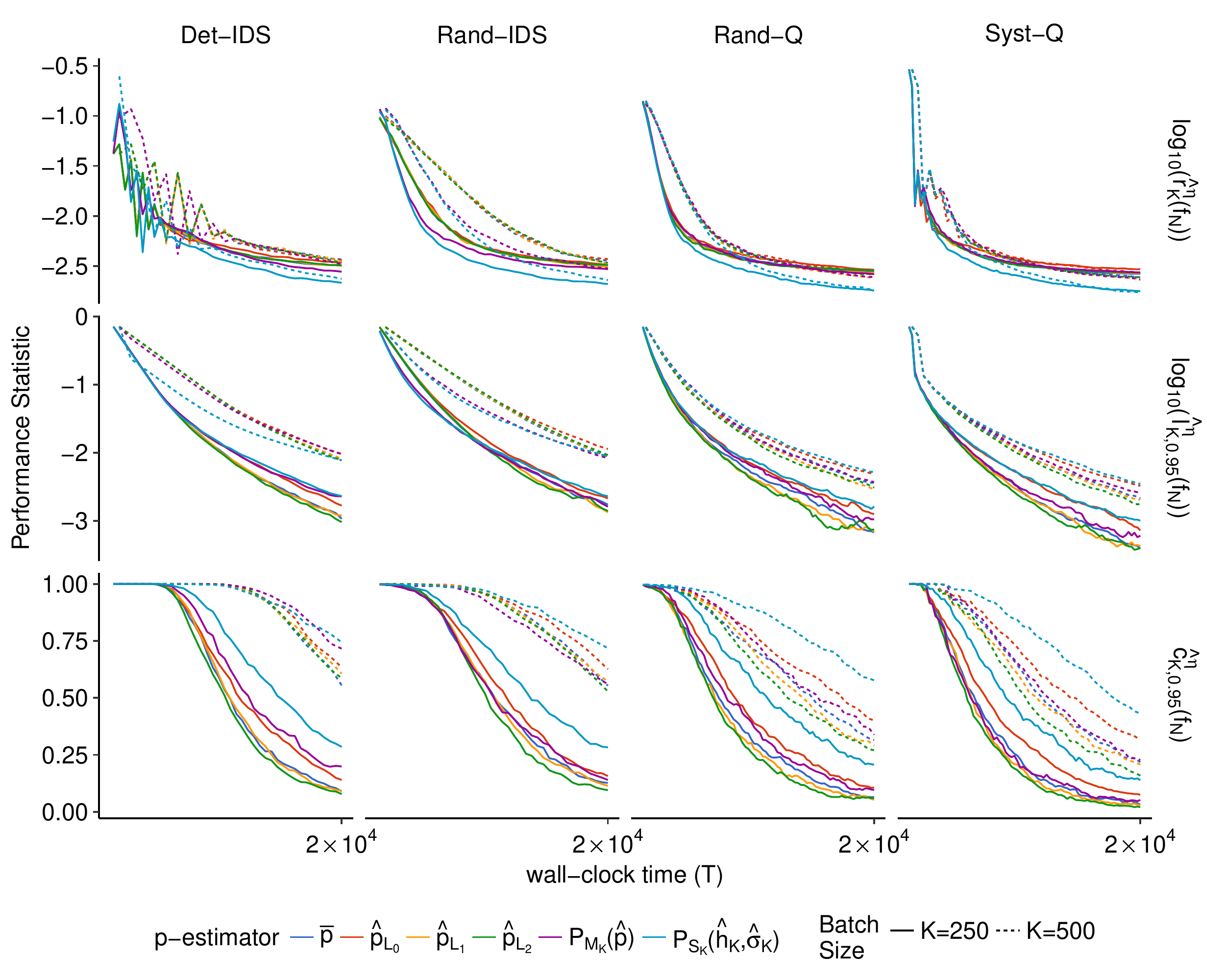}
			\caption{Performance statistics  $\hat{r}_{K}^{\eta}(f_{T}):=m^{-1}\sum_{i = 1}^{m} r_{K}^{\eta}(f_{T}^{(i)})$, $\hat{l}_{K,0.95}^{\eta}(f_{T}):=m^{-1}\sum_{i = 1}^{m}l_{K,0.95}^{\eta}(f_{T}^{(i)})$ (in $\log_{10}$ scale), as well as $\hat{c}^{\eta}_{K,0.95}(f_{T}):=m^{-1}\sum_{i=1}^{m}c_{K,0.95}^{\eta}(f_{T}^{(i)})$ across $i=1,\ldots,m$ Monte-Carlo of each algorithm applied to the linear test function $h_1$ with runs with $m=1,000$. There are four sampling policies $\eta$ (columns), six estimation procedures  $\hat{p}$ (lines), and two batch sizes $K$ (linetype). \label{fig:PerformanceStatistics}}
		}
	\end{figure}

	In terms of the sampling strategies, the Rand-IDS with $K=250$ and Rand-Q  with $K=500$ perform best for minimizing residuals. We observe that all methods struggle with coverage, indicating that $f_T$ prematurely collapses due to ``overconfidence" induced by the bias in $\hat{p}$, cf.~Section~\ref{sub:biasAnalysisPBA}. This effect is mitigated by a larger batch size $K$. Moreover, coverage metrics for IDS methods are higher, primarily driven by the fact that they use fewer macro-iterations (since $N^{IDS} = T/(K \cdot M)$) and hence are less affected by the bias. However, this effect is not practically useful since the IDS methods also have much wider CI's, i.e.,~ they are conservative about $X^*$. We observe that Syst-Q is consistently worse than Rand-Q: they both generate similar absolute residuals, but the CI/coverage of Rand-Q is larger, indicating that it is better in approximating the true posterior $g_n$. Both Randomized  strategies \eqref{eq:random_quantiles_ids} and \eqref{eq:randomized_policy} perform significantly better than the Deterministic counterpart \eqref{eq:quantiles_ids} and \eqref{eq:syst_sampling} for all $\hat{p}$ in terms of minimizing both the absolute residuals and the length of the 95\% CI.
	
	Figure~\ref{fig:PerformanceStatistics} also indicates that the learning rate of the sampling schemes changes over iterations: the randomized methods yield a more rapid reduction in  absolute residuals for $T$ small (i.e.,	~during the first few steps), while the systematic methods enjoy a better asymptotic improvement. This suggests a hybrid heuristic of randomizing the first few macro-iterations (exploring with Rand-Q), and then more aggressively selecting points to maximize entropy reduction (exploiting with Syst-IDS). In all cases, learning is sub-exponential (i.e.,~sub-linear on the log-scale) which occurs due to $p(x) \simeq 0.5$ around the root which necessarily slows down information gains. Indeed, exponential convergence is only feasible when $p(\cdot)$ is bounded away from 1/2. Interestingly, Figure~~\ref{fig:PerformanceStatistics} suggests that the CI of $f_T$ decreases linearly in $T$ which is of course inconsistent with the above slow learning rate of $X^*$ and subsequently ruins coverage, as the mass of the knowledge state $f_T$  no longer includes the true $X^*$.  We observe that the Rand-IDS method is best able to suppress this.
	
	\begin{table}[htb]
		\footnotesize
		\centering
			\caption{MC summary metrics for the test function $h_{1}$ obtained at $T=20,000$. Each of the four  sampling policies $\eta$ is implemented using all six estimation schemes for $p(\cdot)$ with a batch size  $K \in \{250,500\}$. Furthermore, the TPO policy is also included  (last 2 rows); the latter has adaptive $K_\alpha$ with $\alpha \in\{0.05, 0.4\}$.
		}
		\begin{tabular}{c|c|rr|rr|rr}
			\hline
			\multirow{ 2 }{*}{$\eta$} & \multirow{ 2 }{*}{$\hat{p}$} & \multicolumn{2}{c|}{$\hat{r}_{K}^{\eta}(f_{T})$ ($10^{-2}$) }& \multicolumn{2}{c|}{$\hat{l}_{K,0.95}^{\eta}(f_{T})$ ($10^{-2}$)}& \multicolumn{2}{c}{$\hat{c}_{K,0.95}^{\eta}(f_{T})$} (in \%) \\
			& & $K=250$ & $K=500$  & $K=250$ & $K=500$ &  $K=250$ & $K=500$  \\ \hline
			 \multirow{ 6 }{*}{ Det-IDS } & \multirow{ 1 }{*}{ $\bar{p}$ } & 0.3137 & 0.3528 & 0.1346 & 0.8192 & 9.90 & 60.60 \\
			& \multirow{ 1 }{*}{ $\hat{p}_{\mathscr{L}_{0}}$ } & 0.3289 & 0.3623 & 0.1777 & 0.9522 & 15.00 & 65.10 \\
			& \multirow{ 1 }{*}{ $\hat{p}_{\mathscr{L}_{1}}$ } & 0.3132 & 0.3399 & 0.1312 & 0.7723 & 9.80 & 58.20 \\
			& \multirow{ 1 }{*}{ $\hat{p}_{\mathscr{L}_{2}}$ } & 0.3335 & 0.3412 & 0.1200 & 0.6937 & 10.20 & 54.00 \\
			& \multirow{ 1 }{*}{ $\mathscr{P}_{\mathscr{M}}$ } & 0.2634 & 0.3247 & 0.2194 & 0.9516 & 20.30 & 71.60 \\
			& \multirow{ 1 }{*}{ $\mathscr{P}_{\mathscr{S}}$ } & 0.2155 & 0.2365 & 0.2509 & 0.7647 & 29.10 & 71.60 \\
			\cline{1-8}\multirow{ 6 }{*}{ Rand-IDS } & \multirow{ 1 }{*}{ $\bar{p}$ } & 0.3079 & 0.3329 & 0.1655 & 0.9596 & 13.90 & 58.60 \\
			& \multirow{ 1 }{*}{ $\hat{p}_{\mathscr{L}_{0}}$ } & 0.3108 & 0.3585 & 0.2065 & 1.0756 & 16.00 & 63.00 \\
			& \multirow{ 1 }{*}{ $\hat{p}_{\mathscr{L}_{1}}$ } & 0.3196 & 0.3524 & 0.1653 & 0.9793 & 12.40 & 55.90 \\
			& \multirow{ 1 }{*}{ $\hat{p}_{\mathscr{L}_{2}}$ } & 0.3238 & 0.3456 & 0.1500 & 0.8778 & 9.90 & 49.90 \\
			& \multirow{ 1 }{*}{ $\mathscr{P}_{\mathscr{M}}$ } & 0.2867 & 0.3017 & 0.2149 & 0.8888 & 15.70 & 56.00 \\
			& \multirow{ 1 }{*}{ $\mathscr{P}_{\mathscr{S}}$ } & 0.1990 & 0.2248 & 0.2765 & 0.9066 & 29.60 & 74.70 \\
			\cline{1-8}\multirow{ 6 }{*}{ Rand-Q } & \multirow{ 1 }{*}{ $\bar{p}$ } & 0.2653 & 0.2417 & 0.1010 & 0.3105 & 6.70 & 30.10 \\
			& \multirow{ 1 }{*}{ $\hat{p}_{\mathscr{L}_{0}}$ } & 0.2958 & 0.2706 & 0.1420 & 0.4431 & 11.10 & 36.70 \\
			& \multirow{ 1 }{*}{ $\hat{p}_{\mathscr{L}_{1}}$ } & 0.2591 & 0.2680 & 0.0798 & 0.3617 & 6.90 & 30.70 \\
			& \multirow{ 1 }{*}{ $\hat{p}_{\mathscr{L}_{2}}$ } & 0.2688 & 0.2540 & 0.0567 & 0.2844 & 4.80 & 24.90 \\
			& \multirow{ 1 }{*}{ $\mathscr{P}_{\mathscr{M}}$ } & 0.2626 & 0.2447 & 0.1058 & 0.3514 & 8.40 & 34.30 \\
			& \multirow{ 1 }{*}{ $\mathscr{P}_{\mathscr{S}}$ } & 0.1851 & 0.1893 & 0.1623 & 0.4613 & 20.00 & 53.20 \\
			\cline{1-8}\multirow{ 6 }{*}{ Syst-Q } & \multirow{ 1 }{*}{ $\bar{p}$ } & 0.2840 & 0.2368 & 0.0456 & 0.1824 & 3.70 & 19.80 \\
			& \multirow{ 1 }{*}{ $\hat{p}_{\mathscr{L}_{0}}$ } & 0.2943 & 0.2670 & 0.0833 & 0.3268 & 5.80 & 28.60 \\
			& \multirow{ 1 }{*}{ $\hat{p}_{\mathscr{L}_{1}}$ } & 0.2852 & 0.2437 & 0.0388 & 0.2031 & 3.70 & 18.20 \\
			& \multirow{ 1 }{*}{ $\hat{p}_{\mathscr{L}_{2}}$ } & 0.2819 & 0.2380 & 0.0363 & 0.1863 & 2.30 & 17.10 \\
			& \multirow{ 1 }{*}{ $\mathscr{P}_{\mathscr{M}}$ } & 0.2894 & 0.2375 & 0.0526 & 0.2160 & 3.20 & 20.90 \\
			& \multirow{ 1 }{*}{ $\mathscr{P}_{\mathscr{S}}$ } & 0.1775 & 0.1634 & 0.1029 & 0.3439 & 11.10 & 45.60 \\
			\cline{1-8}\multirow{2}{*}{ TPO } & $p_{0.05}$& \multicolumn{2}{c|}{0.8640}& \multicolumn{2}{c|}{4.8232}& \multicolumn{2}{c}{4.0}\\
			& $p_{0.40}$& \multicolumn{2}{c|}{1.2558}& \multicolumn{2}{c|}{39.3445 }& \multicolumn{2}{c}{72.7}\\
			\hline
		\end{tabular}
		\label{tab:sampling-comp-mc}
	\end{table}

	Table~\ref{tab:sampling-comp-mc} lists the final summary statistics at $T=20,000$  for the considered combinations $(\eta,\hat{p},K)$ utilizing the linear test function $h_{1}$. We observe that even for this straightforward setting, a small batch size (cases $K=50$ and $K=100$ which were omitted) is insufficient, leading to a situation whereby absolute residuals are very large whereas the average 95\% CI length is small so that $f_T$ is collapsing prematurely. On the other hand, for $K=500$,  the average absolute residuals, as well as the average 95\% CI length are significantly small across all $\eta$ and $\hat{p}$, indicating that the associated posterior $f_{T}$ is placing most of its mass near  the actual root value $x^{*}$. If $K=250$, then it can be observed that there is a good balance between residuals and length of CI. In particular, \textit{Rand-Q} behaves well in combination with the CLT estimator.
	Finally, the last two rows of Table~\ref{tab:sampling-comp-mc} summarize the performance of TPO-PBA. This policy leads to very large batch sizes, and in this case study used just $N=6$ and $N = 9$ (median) sampling locations  with $\alpha = 0.05$ and $\alpha = 0.40$, respectively. As a result TPO-PBA is not able to learn $X^*$, leading to residuals  as well as the CI length that are significantly larger in comparison to our G-PBA policies.
	
	\subsubsection{Sensitivity Analysis for the Exponential and Cubic Test Functions}
	
	Performance evaluation metrics (Tables only) for the other test cases~\eqref{eq:g-ex2}--\eqref{eq:g-ex3} appear in Appendix A.2. Here we discuss the main take-aways. Among policies, Rand-Q works best for $h_2$ and Syst-Q for $h_3$ although the differences are not significant. As before, the functional response estimator $\mathscr{P}_{\mathscr{S}_{K}}(\hat{h}_{K},\hat{\sigma}_{K})$ performs best for root-finding, yielding lowest estimator error and highest coverage. This confirms the value of using as much information from the oracle as possible. Due to the more difficult setting, a larger batch size $K=500$ (representing a total of $N=40$ sampling locations for the Q-based policies and $N^{IDS} = 20$ for the IDS policies) is needed. Table 5 in the Appendix shows the complete failure of PBA when $K$ is too small ($K=50$ in the Table) whereby $f_T$ collapses, severely underestimating posterior uncertainty and leading to zero coverage.
	
	If only the response sign~\eqref{eq:pba_response} is used to learn $X^{*}$, then there is no clear ``winner'' among the proposed methods. We observe that the IDS policies are less accurate (higher $\hat{r}$) but also higher coverage.  Similarly, the empirical $\bar{p}$ and posterior median $\hat{p}_{\mathscr{L}_{1}}$ are best for maximizing accuracy while the posterior mode $\hat{p}_{\mathscr{L}_{0}}$ is best for maximizing coverage.  This is consistent with previous discussion that the latter minimizes bias in learning $p(x) \simeq 1/2$ and so is a more ``conservative'' approach that slows down error propagation in $f_n$. The majority boosting approach with $\mathscr{P}_{\mathscr{M}}$ also works quite well.
	
	We observe that in these more challenging settings, all methods suffer from model mis-specification which cause $f_N$ to deviate from the true posterior and lead to poor statistical coverage with respect to the true root. This premature posterior collapse ranges from extremely severe ($\hat{r} \gg \hat{l}$ so the residuals are much larger than the estimated uncertainty about $X^*$), to moderate (coverage $\hat{c}_{0.95} \in [0.5,0.8]$). It is present also for the TPO-PBA approach which is supposed to control estimation error for $\hat{p}$ but relies on the assumption of Gaussian noise. Thus, it remains an open problem to find an approach that would guarantee asymptotic consistency, or at least heuristically match the preset coverage levels. For now our results confirm the strong sensitivity of PBA to properly estimating oracle properties and the discrepancy between the generally low residuals  obtained (i.e.~good root estimate) and the mediocre quantification of root uncertainty. 
	
	\subsection{Evaluating the Quality of the Design}\label{sub:qualityDesign}
	
	To focus on the sampling aspect of G-PBA, we examine more closely the design quality of our schemes. To do so, we take $x_{1:n}^{(K,\eta)}$ as the sequence of design points obtained implementing the sampling policy $\eta$ and batch size $K$ for $n=1,\ldots,N-1$ and then compute the true posterior
	$g_{n}^{(K,\eta)}$ based on the obtained sequence $(x,y)_{1:N}$. We then again evaluate our results considering the resulting
	absolute residuals $\hat{r}_{K}^{\eta}(g_{n})$ and length of $(1-\alpha)\%$ CI length, only.
	
	The G-PBA sampling strategies $\eta$ are benchmarked against the following \textit{baseline} schemes which employ the true $p(x)$ (and therefore the actual posterior density $g_{n}$) to select sampling locations:
	
	\begin{enumerate}[label=(\roman*),align=left]
		\item True-IDS: greedily maximizing the  information criterion~\eqref{eq:InformationCriterion} across the full $x_{n+1} \in (0,1)$:
		\begin{equation}
		\label{eq:IDSTrueSampling}
		x_{n+1} = \argmax_{x \in (0,1)} \mathcal{I}(x,g_{n};p(x));
		\end{equation}
		
		\item Median-sampling: proposed in~\cite{waeber2011bayesian} even for the case where $p(\cdot)$ is location-dependent:
		\begin{equation}
		\label{eq:medianSampling}
		x_{n+1} = \text{median}(g_{n});
		\end{equation}
		
		\item Uniform Sampling (Uniform):
		\begin{equation}
		\label{eq:UniformSampling}
		x_{n+1} \sim \mathsf{Unif}(0,1).
		\end{equation}
	\end{enumerate}
	
	Policies \eqref{eq:IDSTrueSampling} and \eqref{eq:medianSampling} sequentially employ $g_{n}$ to select new sampling locations; whereas \eqref{eq:UniformSampling} is a \textit{passive policy} and hence a lower bound for adaptive sampling. Recall that \eqref{eq:IDSTrueSampling} is optimal in the sense of maximizing the expected KL distance between $g_{n}$ and $g_{n+1}$ and thus used as an upper bound on performance. In order to make  these baseline policies comparable to the G-PBA policies \eqref{eq:quantiles_ids}--\eqref{eq:randomized_policy}, $g_{n}$ must also be updated in batch (i.e., using the transition function \eqref{eq:updating_states}).
	
	Figure~\ref{fig:ResidualsActual} shows the average absolute residuals (left panel) and average length of the 95\% $g_{T}$-credible intervals (CI, right panel) for G-PBA  and baseline policies for $K \in \{50,250,500\}$ in wall-clock time $T$. The evaluation of G-PBA is done using three representative estimators for $p(\cdot)$: the empirical majority proportion~\eqref{eq:emp_prop_estimator}, the posterior mode~\eqref{eq:posterior_mode} and the functional aggregation~\eqref{eq:Prob_correct_clt}. As expected the baseline strategies illustrate the extremes. On the one hand, the IDS policy using the true $p(x)$ wins overall, yielding very low residuals and tight CI's. On the other hand, median-sampling produces even smaller residuals but completely fails to reduce uncertainty, returning very large CI (sometimes even worse than from Uniform sampling). Thus, median-sampling is too uninformative and prevents the posterior to shrink to a Dirac mass. Also note that if $p(x)$ is known then batching slows down the learning rate of $X^{*}$. This is best seen in the first few steps: to attain an average absolute residual of $1\times 10^{-2}$ (i.e., -2 in $\log_{10}$ scale) it requires $T =550,2250,4500$ for $K= 50,250,500$, respectively. On the other hand, by $T=20K$ wall-clock iterations the impact of batching is insignificant.
	
	Among G-PBA approaches we see that for $K=50$ they do not perform well, unable to beat even Uniform sampling in terms of their proposed design. This re-iterates the need to accurately learn $p(x)$ to select sensible sampling locations. Det-IDS produces the best design for large $K$ and is best able to approximate the true IDS sampling. Another insight is that the randomized policies are very  successful for low $T$, reducing residuals even faster than IDS with true $g_T$. This implies that a viable heuristic would be to randomize the first few macro-iterations (explore), and then act more aggressively so as to select points at which the entropy reduction is maximal (exploit). Finally, we again confirm the superiority of using functional responses, which in particular best reduce CI length of $X^*$ (right panel of Figure~\ref{fig:ResidualsActual}).

	\begin{figure}[htb]
		{
			\centering
			\includegraphics[width=.90\textwidth]{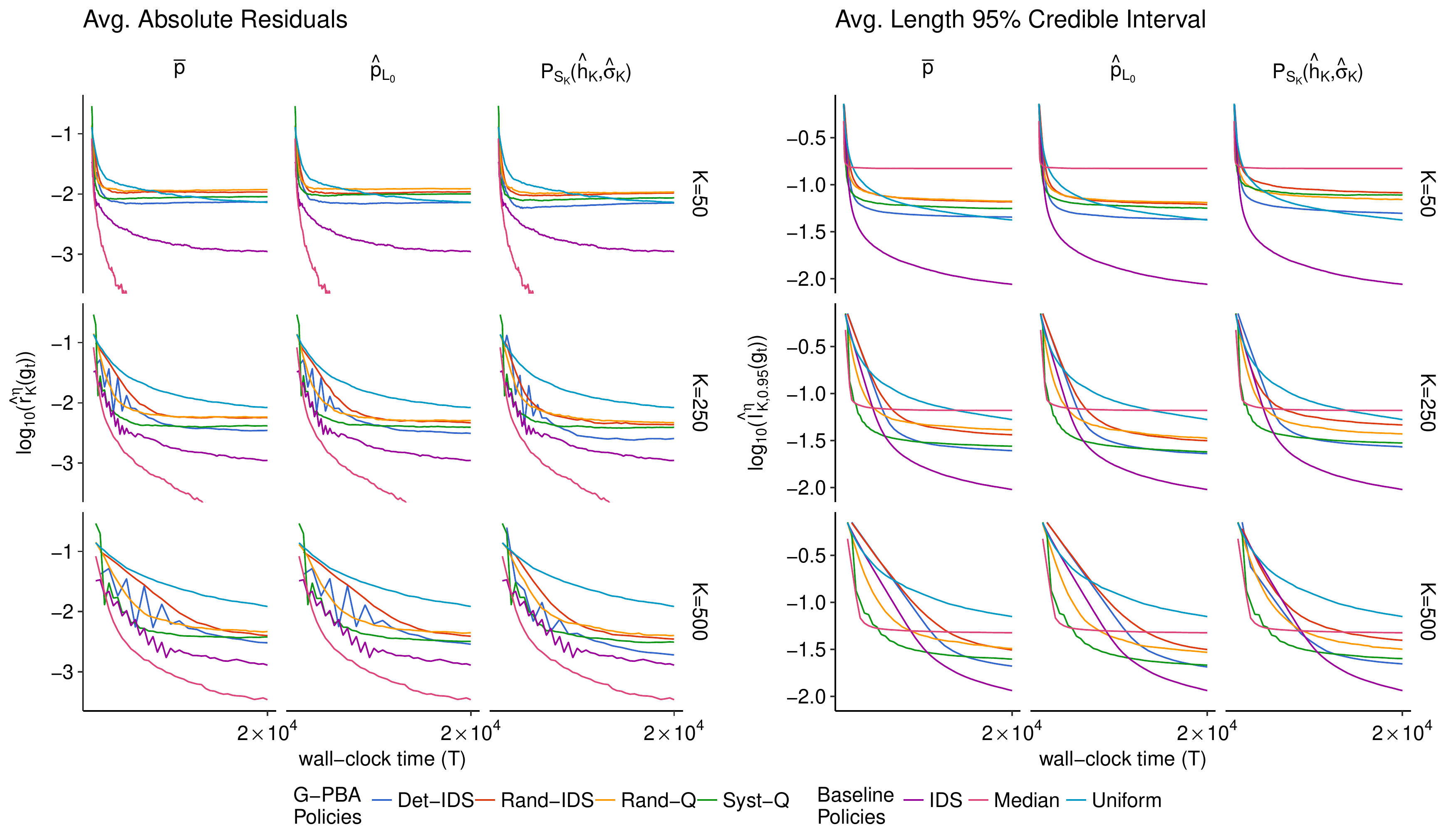}
			\caption{Average absolute residuals ($y$-axis on left panel) and average length of 95\% CI ($y$-axis on right panel) evaluated in wall-clock time ($x$-axis) obtained utilizing the sampling points generated by the G-PBA policies and evaluate them into the updating model that uses the true $p(\cdot)$; as well as the baseline sampling policies for batch sizes $K\in \{50,250,500\}$ (rows) and fixed estimation methods for $p(\cdot)$ (columns). The baseline methods are independent of $\hat{p}$ and hence same across columns.
				\label{fig:ResidualsActual}}
		}
	\end{figure}
	
	\section{Case Study: Root-Finding for Optimal Stopping}
	\label{sec:osp}
	
	Let us briefly recall a generic discrete-time optimal stopping problem on a finite horizon. Let $X \equiv X_{1:\bar{T}}$ be a  real-valued Markov process generating an information filtration $\mathcal{G} = \sigma(X_{1:t})$. Set $\mathcal{S}$ to be the collection of all $\mathcal{G}$-stopping times smaller than some given horizon $\bar{T}<\infty$, and $H(t,x)$ the (bounded) reward function for stopping at time $t=0,1,\ldots,\bar{T}$. The {Optimal Stopping Problem} (OSP) consists of maximizing the expected reward $H(\tau, X_{\tau})$ over $\tau \in \mathcal{S}$. Towards solving the OSP,  define the \textit{value function}
	$V(t,x) := \sup_{\tau\geq t, \tau \in \mathcal{S}} \mathbb{E}\left[H(\tau, X_{\tau})|X_{t} = x \right]$ for any $0\leq t\leq \bar{T}$. Standard dynamic programming arguments imply that $V(t,x) = H(t,x) + \max \{h(x; t),0\}$ where the function
	\begin{equation}
	\label{eq:time_value}
	h(x; t):= \mathbb{E} \bigl[V(t+1,X_{t+1}) \, |X_{t}=x \bigr] - H(t,x),
	\end{equation}
	is the \textit{timing value}. It follows that the stopping decision at a given $(t,x)$ is equivalent to comparing $V(t,x)$ and $H(t,x)$: $\mathfrak{S}_{t} := \{x : \tau^{*}(t,x) = t\} =   \{x: V(t,x) = H(t,x) \} = \{h(x; t)\leq 0 \}$. Thus, it is optimal to stop immediately if and only if the conditional expectation of tomorrow's reward-to-go is less than the immediate reward. Frequently, a priori structure implies that the stopping set $\mathfrak{S}_t$ above is a half-line, i.e.,~$h(\cdot; t)$ has a unique root $x^{*}$. Consequently, solving the OSP at stage $t$ is equivalent to a root-finding problem for $h(\cdot; t)$.
	
	A stochastic simulation approach (known in the literature as the Longstaff-Schwartz paradigm) recursively builds noisy simulators for $h(t,x)$ over $t=\bar{T}-1,\bar{T}-2,\ldots$. This is obtained by generating forward paths $x_{t:\bar{T}}$ of the state process and computing corresponding path-wise stopping times $\tau \equiv \tau(t+1,x_{t:\bar{T}})$ (which rely on $\mathfrak{S}_{t+1:\bar{T}}$ and hence are recursively known). The realization $z_{t}(x_t) := H({\tau},x_{\tau}) - H(t,x_{t})$ is the pathwise timing value, i.e.,~difference between future and immediate reward over the given trajectory. By construction, $\mathbb{E}[Z_{t}(x_{t})] = h(x_t;t)$ which matches the structure of the oracle \eqref{eq:pba_simulator}. The random component $\epsilon(x; t)$ arises intrinsically from the randomness in the trajectory $x_{t:\bar{T}}$. Therefore, the PBA approach offers a novel algorithm to solve one-dimensional optimal stopping problems. Notably, it essentially bypasses standard value-function approximation methods, and allows to directly quantify accuracy of estimated policy $\widehat{\mathfrak{S}_t} = [0,\hat{x}]$.
	
	As an illustration, we revisit the popular example of a Bermudan Put option within a discretized Black-Scholes model: the reward function is $H(t,x) := e^{-rt}(K^{Put} - x)_{+}$, $(X_t)$ is a log-normal random walk and $r$ is the interest rate. It is well-known that there is a unique \textit{exercise boundary} $x^*(t) \le K^{Put}$, and one should exercise as soon as $X$ drops below this boundary: $\mathfrak{S}_t = [0,x^*(t)]$.
	
	For the numeric example below we take the parameters $K^{Put} = 40, r=0.06, \bar{T}=1$ and restrict to the domain  $(25,40)$ (which is based on some mild domain knowledge, as very low stock prices are known to definitely trigger exercise). Thus, we consider the following oracle (with $t$ fixed):
	
	\begin{equation}
	\label{eq:ospOracle}
	Z^{Put}(x) := h(x;t) + \epsilon(x; t), \qquad x \in [25,40];
	\end{equation}
	where the  latent function $h(\cdot; t)$ is the timing value and $x$ is the stock price at date $t$. The left panel of Figure~\ref{fig:Regression_spline} shows an estimate $\hat{h}(\cdot; t)$ of ~\eqref{eq:time_value}, as well as the distribution of $\epsilon(\cdot; t)$. The plot was obtained by fitting an off-line smoothing spline model to $500$ pointwise estimates $\hat{h}(x_i; t)$ (equidistant in $(25,40)$), each obtained from $K=20,000$ oracle calls, i.e.,~a total of $T=10^7$ function evaluations. A deterministic root finding procedure (Newton-Raphson) was run to estimate $x^*\simeq 35.1249$ (vertical dashed line) based on the latter $\hat{h}$. This estimate of the root is used as the ground truth in the sequel, although notably it comes without any standard error, being based on a point estimate of $h(\cdot; t)$.
	
	In this case study, \eqref{eq:ospOracle} violates the basic PBA assumption of a symmetric noise distribution. Instead Figure~\ref{fig:Regression_spline} demonstrates that $\epsilon(x)$ is right-skewed and heavy-tailed. In particular, $p(x^*) < 0.5$, see right panel of the Figure. Because PBA in fact searches for the point $x^*_{med}$ such that $p(x^*_{med}) = 0.5$, direct use of \eqref{eq:ospOracle} will return the root $x_{med}^*$ of the \emph{median} $\hat{q}_{Z}^{0.50}(x):=\hat{F}_{Z}^{-1}(0.50;x)$ (black line in the Figure) rather than the root $x^*$. To resolve this, we use a \emph{pre-averaging} procedure that considers the sign of an average of $a>1$ oracle
	evaluations:
	\begin{equation}
	\label{eq:PreAveragedOSOracle}
	\bar{Y}^{Put}_{a}(x):= \sign \{\bar{Z}_{a}(x)\}, \quad \bar{Z}_{a}(x) := \textstyle a^{-1}\sum_{l=1}^{a}Z_{l}(x).
	\end{equation}
	The G-PBA now works with \eqref{eq:PreAveragedOSOracle} to estimate
	the corresponding probability of correct response
	\begin{equation}
	\label{eq:ProCorrectPreAverageOracle}
	p_{a}^{Put}(x):= \mathbb{P}(\bar{Y}^{Put}_{a}(x) = \sign\{x^{*}-x\}),
	\end{equation}
	as before, namely by using a batch of $K^{\prime}$ and then considering the signal $B^{Put}_{a,K^{\prime}}(x) := \sum_{j = 1}^{K^{\prime}}1_{\{\bar{Y}^{Put}_{a,j}(x) = +1\}}$. Denoting by $K$ the total number of oracle queries at $x$ we have $K^{\prime} = K/a$ for the number of replicates to query $\bar{Y}^{Put}_{a}(x)$. Note that pre-averaging is not needed for functional response aggregation.
	
	The principal role of $a$ is to alleviate the skewness of $\epsilon(x)$. The right panel of Figure~\ref{fig:Regression_spline} shows $x \mapsto p_{a}^{Put}(x)$ for a range of $a$'s. We see that for $a \le 10$ we still have the difficulty that $x_{med}^*$ is far from $x^*$. In contrast, for $a \in \{25,50\}$ the Figure shows that the PBA assumptions are satisfied: $p_{a}^{Put}(x)>1/2$ for all $x\neq x^{*}$ and $p_{a}^{Put}(x^{*})\simeq 1/2$. Pre-averaging also has the side effect of boosting the signal-to-noise ratio and hence boosting $p_a^{Put}$ similar to the majority-vote estimator in \eqref{eq:px_majority_2}. Overall, the choice of $a$ is governed by the above aim of making $\epsilon$ symmetric, as well as the trade-off between sampling many locations to find $x^*$ vis-a-vis proper probabilistic updating of $f_n$. Our analysis suggests to take $a$ as small as feasible and keep $K^{\prime}$ relatively large. Lastly, since our earlier analysis assumed a decreasing response, the sign of $Z^{Put}(x)$ is flipped in the sequel.
	
	\begin{figure}[htb]
		{
			\centering
			\includegraphics[width=0.85\textwidth]{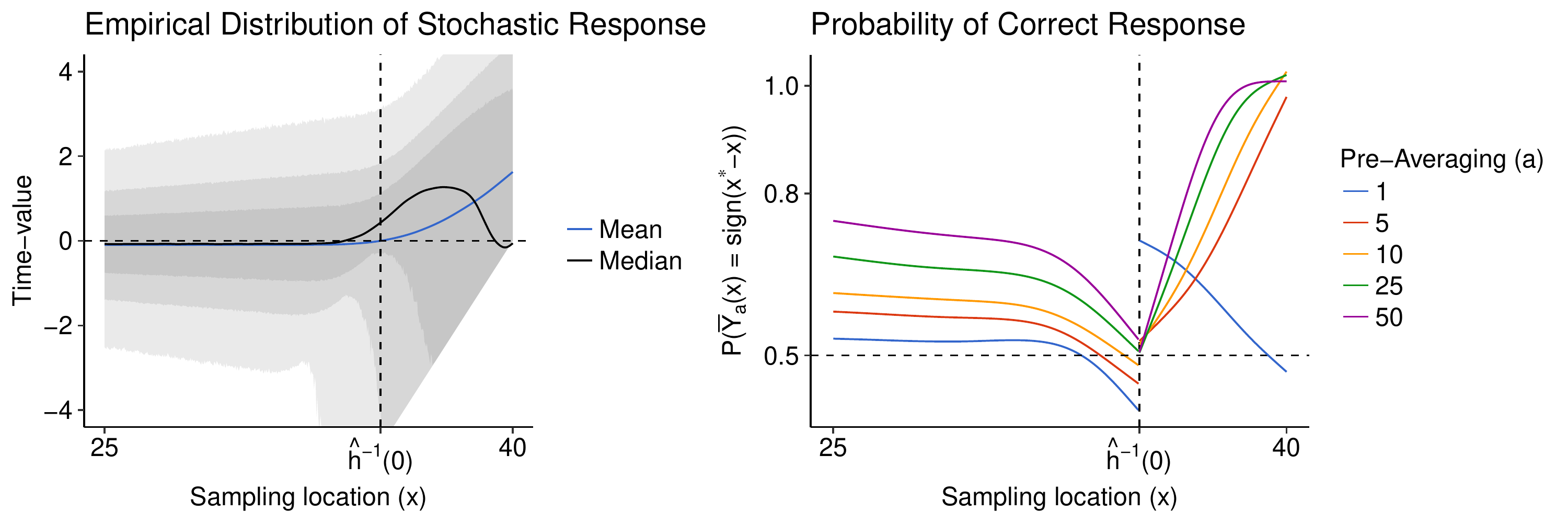}
			\caption{\emph{Left} panel: Bermudan Put oracle distribution: the  fitted mean response $\hat{h}$ (blue line), median $\hat{q}_{Z}^{0.50}(x)$ (black line) and empirical quantiles (at 1\%, 10\%, 25\%, 75\%, 90\% and 99\% levels, different shaded areas). The root estimate $\hat{h}^{-1}(0; t) \simeq 35.125$ (dashed vertical line) is obtained using Newton-Raphson on the off-line surrogate~$\hat{h}$. \emph{Right}: estimated probability of correct response for the pre-averaged oracle $\bar{Y}_{a}^{Put}(\cdot)$ for $a \in \{1,5,10,25,50\}$. Note that $p_{a}^{Put}(x) < 0.5$ is possible due to skewed noise. \label{fig:Regression_spline}}
		}
	\end{figure}

	\subsection{G-PBA for Optimal Stopping}
	\label{sub:GPAforOSP}
	
	We implement the G-PBA sampling policies $\eta$  combined with three best estimation procedures based on Section~\ref{sub:SensitivityAnalysis}~and~\ref{sub:qualityDesign}:  the empirical proportion, $\bar{p}$ and the posterior mode, $\hat{p}_{\mathscr{L}_{0}}$ using the pre-averaged simulator~\eqref{eq:PreAveragedOSOracle} with $a=25$ (smallest pre-averaging value which controls for skewness), $K \in \{1000,2000\}$; as well as the $Z$-based $\mathscr{P}_{\mathscr{S}}(\hat{h},\hat{\sigma})$ with  $a=1$ and $K \in \{1000, 2000\}$.
	Table~\ref{tab:mc-osp} shows the  performance statistics~\eqref{eq:EvalMeasure-3}, \eqref{eq:EvalMeasure-1} and \eqref{eq:EvalMeasure-2} after $T=20,000$  wall-clock iterations based on 1,000 MC macro-replications.
	
	\begin{table}[ht]
		\footnotesize
	\caption{Average performance statistics at $T=20,000$ using 1,000 MC macro-replications. Three estimators for $p_{a}^{Put}(\cdot)$ are implemented using all G-PBA sampling policies.\label{tab:mc-osp}}
	\centering
	\begin{tabular}{c|c|rr|rr|rr}
		\hline
		\multirow{ 2 }{*}{$\eta$} & \multirow{ 2 }{*}{$\hat{p}$} & \multicolumn{2}{c|}{$\hat{r}_{k}^{\eta}(f_{T})$ }& \multicolumn{2}{c|}{$\hat{l}_{K,0.95}^{\eta}(f_{T})$ }& \multicolumn{2}{c}{$\hat{c}_{k,0.95}^{\eta}(f_{T})$ (in \%)}\\
		& & $K=1000$ & $K=2000$  & $K=1000$ & $K=2000$  & $K=1000$ & $K=2000$  \\
			\cline{1-8}\multirow{ 3 }{*}{ Det-IDS } & \multirow{ 1 }{*}{ $\bar{p}$ } & 0.423 & 0.711 & 1.204 & 3.164 & 67.30 & 95.70 \\
			 & \multirow{ 1 }{*}{ $\hat{p}_{\mathscr{L}_{0}}$ } & 0.415 & 0.735 & 1.379 & 3.271 & 71.40 & 95.80 \\
			 & \multirow{ 1 }{*}{ $\mathscr{P}_{\mathscr{S}}$ } & 0.250 & 0.621 & 1.358 & 3.194 & 92.90 & 99.60 \\
			\cline{1-8}\multirow{ 3 }{*}{ Rand-IDS } & \multirow{ 1 }{*}{ $\bar{p}$ } & 0.432 & 0.675 & 1.439 & 3.547 & 69.10 & 94.70 \\
			  & \multirow{ 1 }{*}{ $\hat{p}_{\mathscr{L}_{0}}$ } & 0.422 & 0.675 & 1.622 & 3.680 & 74.70 & 95.10 \\
			  & \multirow{ 1 }{*}{ $\mathscr{P}_{\mathscr{S}}$ } & 0.246 & 0.420 & 1.515 & 3.012 & 91.00 & 98.20 \\
			\cline{1-8}\multirow{ 3 }{*}{ Rand-Q } & \multirow{ 1 }{*}{ $\bar{p}$ } & 0.349 & 0.406 & 0.899 & 1.923 & 56.40 & 81.70 \\
			  & \multirow{ 1 }{*}{ $\hat{p}_{\mathscr{L}_{0}}$ } & 0.375 & 0.430 & 1.119 & 2.149 & 62.30 & 85.70 \\
			  & \multirow{ 1 }{*}{ $\mathscr{P}_{\mathscr{S}}$ } & 0.184 & 0.271 & 0.943 & 2.088 & 81.70 & 97.60 \\
			\cline{1-8}\multirow{ 3 }{*}{ Syst-Q } & \multirow{ 1 }{*}{ $\bar{p}$ } & 0.321 & 0.438 & 0.695 & 1.688 & 46.90 & 78.50 \\
			  & \multirow{ 1 }{*}{ $\hat{p}_{\mathscr{L}_{0}}$ } & 0.399 & 0.504 & 0.991 & 1.917 & 56.00 & 83.70 \\
			  & \multirow{ 1 }{*}{ $\mathscr{P}_{\mathscr{S}}$ } & 0.172 & 0.249 & 0.834 & 1.889 & 80.70 & 96.40 \\
			\hline
		\end{tabular}
	\end{table}
	
	Due to the non-standard noise component and very low signal-to-noise ratio, this is a difficult root-finding problem, comparable to test case $h_3$; in particular the simulation budget $T=20,000$ is quite low. Table~\ref{tab:mc-osp} conveys two main findings. First, we see that the Q-policies perform notably better than the IDS-based policies, although their coverage is a bit worse. This can be interpreted as a preference for exploration, i.e.,~better performance when the number of sampling locations $N$ is larger. It is also observed in the better results for $K=1,000$ vis-a-vis $K=2000$ (however we found that $K=500$ performs noticeably worse).
	Second, we continue to observe dramatic improvement from using functional responses ($\mathscr{P}_{\mathscr{S}_{K}}$) relative to the signed responses ($\hat{p}$-estimators). In this example this effect is amplified since to estimate $\hat{p}$ we must first pre-average the $Y$-values which lowers the quality of estimation since $B^{Put}(x_{n})$ can only use $K/a$ replicates at $x_{n}$; in turn functional aggregation can directly utilize all of the $K$ queries.
	
	We do not observe significant differences among various procedures to obtain $\hat{p}$ or between the Randomized and Systematic sampling. The best overall method is  Rand-Q policy with $\mathscr{P}_{\mathscr{S}_{K}}$ estimator and $K=1,000$ (so $N=20$ updating rounds for $f_N$) which gives average absolute error of 0.17 with a credible interval of 0.83.
	
	\subsubsection{Sampling locations}
	
	Figure~\ref{fig:DistSamplingPointsAmOption} shows the empirical distribution of the sampling locations $x_{1:n}^{(K,\eta)}$ across sampling policies $\eta$ (rows) combined with  $\bar{p}$ using  $K=1,000$.
	As expected, at later stages, all policies sample close to the root $x^*$ (dotted horizontal line). At the same time, there is a significant difference between systematic and randomized schemes early on when $n$ is small. The two deterministic policies \eqref{eq:quantiles_ids} and \eqref{eq:syst_sampling} sample at locations which are relatively far from the root $x^{*}$ at early stages, but then begin to sample very close. This is confirmed by the heavy tails at both extremes of the distribution of $x_{n}^{(K,\eta)}$. Note  the slower convergence of the average estimated root $\hat{x}_{n}^{(K,\eta)}$ (solid line)  to the actual root $x^{*}$ (dotted line). In contrast, the randomized policies  \eqref{eq:random_quantiles_ids} and \eqref{eq:randomized_policy} sample close to $x^*$ already for small $n$, but have larger spread (i.e.,~more exploration) throughout the iterations. The latter is conducive to quickly reducing absolute residuals after a few samples.
	We also observe that the quantile-based policies have lower variability in $f_N$ partly due to utilizing more locations to learn $X^*$ (while the IDS strategies discard half their data). Within the former, we see that the \textit{Rand-Q} policy \eqref{eq:randomized_policy} appears to perform best. Relative to other strategies,  Syst-Q displays a nearly deterministic pattern of sampling locations, i.e.,~it is minimally adaptive to the observations $Y_k$ and produces a very distinct oscillating ``bracketing'' sequence of $x_{1:n}$ across different runs. This systematic approach lowers variance of $\hat{x}_N$ but also tends to affect accuracy.
	
	\begin{figure}[htb]
		{
			\centering
			\includegraphics[width=0.80\textwidth]{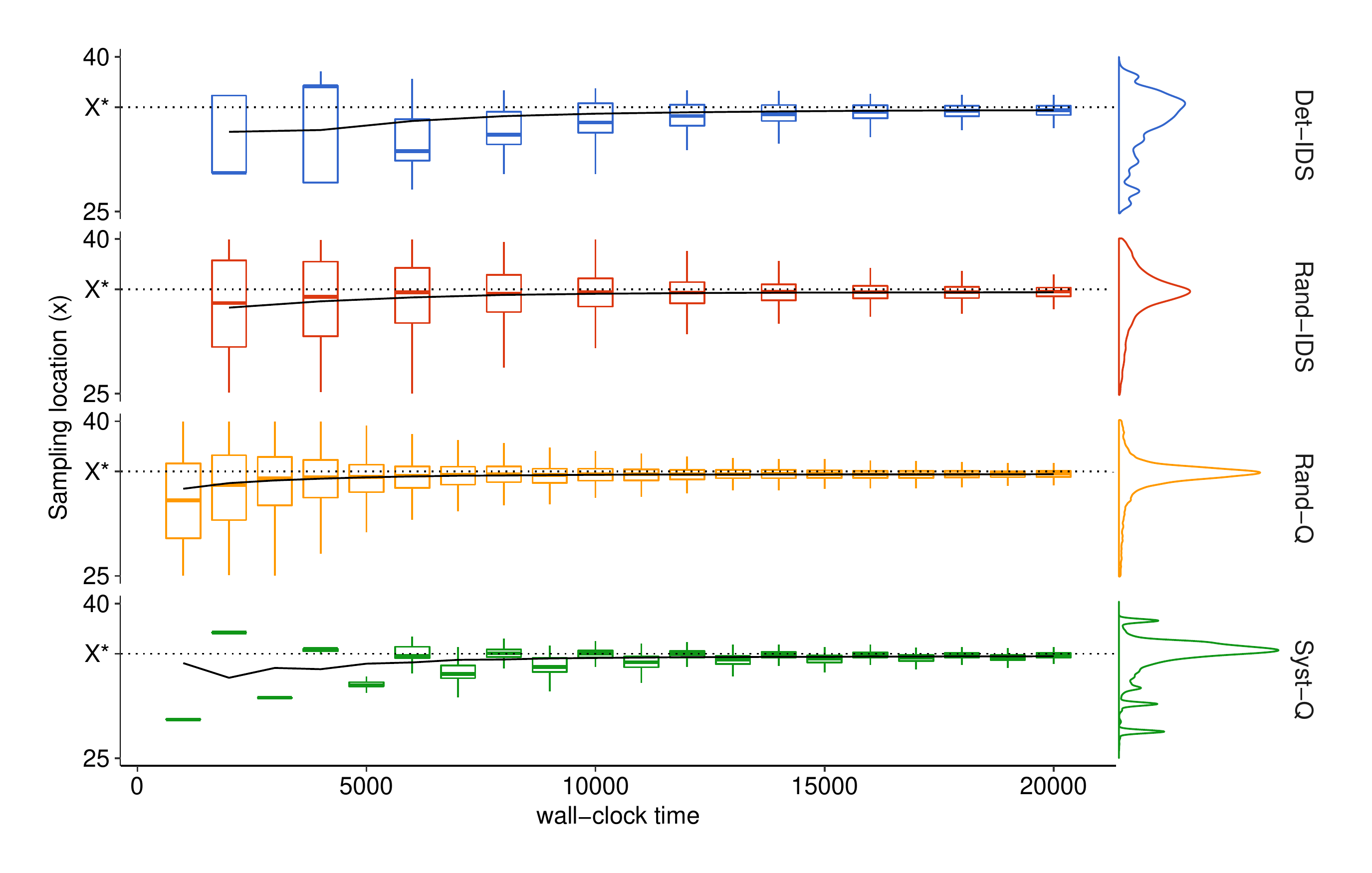}
			\caption{Box-plot of sampling locations $x_{n}^{(K,\eta)}$ ($y$-axis)  across sampling policies $\eta$ (rows) evaluated in wall-clock time($x$-axis) for  $K=1,000$ and $\bar{p}$, obtained after 1,000 MC macro-iterations. The solid line represents the average root estimate $\hat{x}_{n}^{(K,\eta)}$ and the dashed line the root $x^{*}\simeq 35.1249$. \label{fig:DistSamplingPointsAmOption}}
		}
	\end{figure}

	\section{Conclusion}
	\label{sec:conclusion}
	We have developed a family of numerical schemes that generalize probabilistic bisection to the setting where oracle accuracy is unknown and location-dependent. Two take-aways are the advantage of quantile sampling against Information Gain approaches, and the emphasis on reliably estimating $p(\cdot)$. For the former, we note that selecting locations according to $f_n$ minimizes the need to know $p(\cdot)$ during the design construction and hence makes more efficient use of oracle queries.  We also observe that randomized sampling (mimicking the successes of Thompson sampling in other learning contexts)  is an efficient heuristic for balancing the exploration/exploitation trade-off, and improves the learning rate in the early stages by better exploring the posterior of $X^*$.
	
	As discussed, a good estimate of $p(x)$ is invaluable in every step of G-PBA, from updating the knowledge state to evaluating information gain. One implication is the critical role of batched sampling, which in particular calls for surprisingly large batch sizes $K$. We  also document in our experiments the significant advantage in using functional responses (via functional aggregation) relative to utilizing just the signs of the responses.  Otherwise, empirical majority proportion or a conservative Bayes-like $\hat{p}_{\mathscr{L}_0}$ are good choices.

	Note that all the proposed estimators were  constructed locally at $x_n$ and did not use information from previous locations $x_{1:n-1}$. As such, they are robust to arbitrary specification of $p(x)$ and can be viewed as making minimal assumptions about the oracle. A structured extension would be to construct a \textit{spatial} model for $p(\cdot)$ incorporating the knowledge acquired at previous sampling locations $x_{1:n}$. For example, this could be achieved by regressing the observed empirical proportions $\bar{p}(x_{1:n})$ on $x_{1:n}$. Such a blend of a regression-type paradigm (common e.g.~for maximizing $h(\cdot)$) with PBA will be explored in a companion forthcoming paper.

\section*{Acknowledgments}
\label{sec:Acknowledgments}

Rodriguez is partially supported by the National Science and Technology Council of Mexico (CONACYT) and University of California Institute for Mexico and the United States (UCMEXUS) under grant CONACYT-216011. Ludkovski is partially supported by NSF DMS-1521743. We thank Peter Frazier for several helpful discussions and introducing us to PBA.

\appendix

\section{Appendix}

\subsection{Additional Results and Proofs}
\label{sub:ProofsAndResults}

\begin{proof}[Theorem~\ref{thm:batched_sampling}]
We will show that the updating equations~\eqref{eq:batched_updating_pba} hold for any $K \in \mathbb{N}$ via mathematical induction. To do so, let $x$ be a fixed sampling location in the support of the knowledge state variable $g_{n}$ and define $B_{K} := \sum_{j=1}^{K} 1_{\{Y_{j} = +1\}}$ to be the total number of observed positive signs after querying the oracle $K$ times at $x$, where we have dropped the dependency on $x$ in $Y_{j}(\cdot)$. Re-expressing the knowledge transition function~\eqref{eq:batched_updating_pba} using indicator functions as (disregarding the normalizing constant $c_{n}(x)$ in~\eqref{eq:constant_batched}):
\[
\textstyle g_{n+K}(u)  \propto \left[ \sum_{j=0}^{K} p^{j} (1-p)^{K-j}1_{\{B_{K}=j\}} \right]g_{n}(u)1_{\{u \geq x\}}  +\left[ \sum_{j=0}^{K}  (1-p)^{j} p^{K-j} 1_{\{B_{K}=j\}} \right]g_{n}(u) 1_{\{u < x\}},
\]
with $p\equiv p(x)$, we will prove that~\eqref{eq:batched_updating_pba} holds for any $K$ and fixed $n$. For $K=1$  we have  $\{B_{1}=1\} = \{Y_{1} = +1\}$ and $\{B_{1}=0\}~= \{Y_{1} = -1\}$ and \eqref{eq:batched_updating_pba} corresponds to the updating in \eqref{eq:pba}. We now concentrate on the case $u> x$ and inductively suppose Equation~\eqref{eq:batched_updating_pba} holds for $K$; we now establish it for $K+1$:
\begin{align*}
\textstyle  g_{n+(K+1)}(u)  &\propto  \left[ (1-p)1_{\{Y_{K+1}=-1\}}  + p \cdot 1_{\{Y_{K+1}=+1\}} \right]g_{n+K}(u) \\
			&=  \left[ (1-p)1_{\{Y_{K+1}=-1\}}  + p \cdot 1_{\{Y_{K+1}=+1\}} \right]  
  \times \textstyle \left[ \sum_{j=0}^{K} p^{j} (1-p)^{K-j}1_{\{B_{K}=j\}} \right]g_{n}(u) \\ 
& =: (A_1 + A_2) g_n(u) 
\end{align*}

We now have
\begin{equation*}
A_1 = \textstyle \sum_{j^{\prime}=0}^{K} p^{j^{\prime}+1} (1-p)^{K-j^{\prime}}1_{\{B_{K}=j^{\prime},Y_{K+1}=+1\}} = \textstyle \sum_{j=0}^{K+1} p^{j} (1-p)^{K+1-j}1_{\{B_{K}=j-1,Y_{K+1}=+1\}}
\end{equation*}
Similarly we obtain $A_2 =\textstyle  \sum_{j=0}^{K} p^{j} (1-p)^{(K+1)-j}1_{\{B_{K}=j,Y_{K+1}=-1\}},
$ and
\begin{align*}
 A_1+ A_2 &= \textstyle \sum_{j=0}^{K+1} p^{j} (1-p)^{(K+1)-j} \textstyle \Bigl[1_{\{B_{K}=j,Y_{K+1}=-1\}} + 1_{\{B_{K}=j-1,Y_{K+1}=+1\}}\Bigr] \\
&= \textstyle  \sum_{j=0}^{K+1} p^{i} (1-p)^{(K+1)-j}1\{B_{K+1}=j\}.
\end{align*}
Analogous argument works for $u < x$.
\end{proof}

\begin{lemma}[Likelihood function of the Majority Proportion Estimator]
\label{lemma:likelihood_maj_prop}
Let $x$ be a fixed sampling location at which the oracle is queried $K\ge 2$ times and $B_{K}$  the number of positive values observed in the random sample $Y_{1:K}(x)$. Then, the likelihood function of the majority proportion estimator, $\bar{p}(B_{K}):= \max\{B_{K}/K,1-B_{K}/K\}$, in $p$ is given by~\eqref{eq:P-likelihood}. \end{lemma}
\begin{proof}
Given $B_{K}\sim\mathsf{Bin}(K,p^{+})$ and $p^{+}(x) := \mathbb{P}(Y(x) = +1)= p(x) 1_{\{ x^* \le x\}} + (1-p(x)) 1_{\{x^* > x\}}$ for $x\in (0,1)$, we have
\begin{align*}
\mathbb{P}_{p}(\bar{p}(B_{K}) = j/K) &:= \mathbb{P}_{p}(\max\{B_{K}/K,1-B_{K}/K\} = j/K) \\
&= \mathbb{P}_{p^{+}}(B_{K} = j) + \mathbb{P}_{p^{+}}(B_{K} = K-j) \\ 
&= \textstyle  {K \choose j} (p^{+})^{j}(1-p^{+})^{K-j} + {K \choose K-j} (p^{+})^{K-j}(1-p^{+})^{j},   \quad \forall j = 0,1,\ldots,\ceil{K/2}-1;
\end{align*}
which is the sum of two Binomial densities. Finally, if $j=k/2$ then
$ \mathbb{P}_{p}(\bar{p}(B_{K}) = 1/2) = \mathbb{P}_{p^{+}}(B_{K} = K/2)$ which is a single binomial density.
\end{proof}

\begin{proof}[Proof of Theorem \ref{thm:posterior_p_pk}]
\label{proof:posterior_p_pk}
\begin{align*}
\pi(p|\bar{p}(B_{K}) = j/K) &\propto \mathbb{P}_{p}(B_{K} = j) \pi_{0}(p) \\
&\propto 
\textstyle  {K \choose j}[p^{j}(1-p)^{K-j} +(1-p)^{j}p^{K-j}]1_{(1/2,1)}(p), & \mbox{if $j=0,1,\ldots, (\ceil{K/2}-1)$;}	
\end{align*}
with the normalizing constant $\beta_{1} = \int_{1/2}^{1} \left[p^{j}(1-p)^{K-j} + (1-p)^{j}p^{K-j}\right] dp$
which can be expressed in terms of the Beta function.
\end{proof}

\begin{corollary}
\label{cor:posterior_mean}
The posterior mean $	\hat{p}_{\mathscr{L}_{2}}(j/K):= \mathbb{E}^{p}_{B_{K}}[p|\bar{p}(B_{K})=j/K]$ is given by
\begin{itemize}
\item If $j=0,1,\ldots,(\ceil{K/2}-1)/K$, then
\begin{align*}
\hat{p}_{\mathscr{L}_{2}}(j/K) &:= \textstyle
\beta_{1}^{-1}\left\{B(j+2,K-j+1)(1-\int_{0}^{1/2}\mathsf{Beta}(p; j+2,K-j+1))dp \right.\\
&\left.\quad  \textstyle  + \  B(K-j+2,j+1)(1-\int_{0}^{1/2}\mathsf{Beta}(p;(K-j+2,j+1)dp\right\}.
\end{align*}

\item If $j=K/2$, then
\[
\hat{p}_{\mathscr{L}_{2}}(j/K) = \textstyle
\frac{\left\{B(K/2+2,K/2+1)(1-\int_{0}^{1/2}\mathsf{Beta}(p; K/2+2,K/2+1))dp \right\} }{B(K/2+1,K/2+1)[1-\int_{0}^{1/2}\mathsf{Beta}(p; K/2+1,K/2+1)dp]}.
\]
\end{itemize}

\end{corollary}

\begin{lemma}\label{prop:bias_p} We have that the \textit{bias} of of the majority proportion estimator $\bar{p} = \bar{p}(B_{K})$ given $p$ is
\begin{equation}
\label{eq:bias_p}
Bias_{p}(\bar{p}) := \mathbb{E}_{p}^{B_{K}}[\bar{p} - p] = \mathbb{P}(B_{K} \leq \ceil{K/2}-1)
-2p \mathbb{P}(B_{K-1} \leq \ceil{K/2} - 2) > 0 , K\geq 3
\end{equation}
\end{lemma}

\begin{proof}
For brevity, we drop the dependency on $x$. 
\begin{align}\notag
\mathbb{E}_{p}[\bar{p}(B_{K})] &:= \mathbb{E}_{p}[\max\{B_{K}/K, 1- B_{K}/K\}]\\ \notag
&= \textstyle  \frac{1}{K}\left\{\mathbb{E}_{p^{+}}^{B_{K}}[B_{K}1_{\{B_{K}\geq \ceil{K/2}\}}] + \mathbb{E}_{p^{+}}^{B_{K}}[(K-B_{K})1_{\{B_{K} < \ceil{K/2}\}}] \right\} \\ \notag
&= \textstyle \frac{1}{K}\left\{\mathbb{E}_{p^{+}}^{B_{K}}[B_{K}] + K\mathbb{P}_{p^{+}}[B_{K} < \ceil{K/2}]
-2\mathbb{E}_{p^{+}}^{B_{K}}[B_{K}1\{B_{K} < \ceil{K/2}\}] \right\} \\ \label{eq:bias-1}
&=\textstyle  p+ \mathbb{P}_{p^{+}}(B_{K} < \ceil{K/2})
-\frac{2}{K}\mathbb{E}_{p^{+}}\left[B_{K}1_{\{B_{K} < \ceil{K/2}\}} \right].
\end{align}
The last term is equal to
\begin{align*}
\mathbb{E}_{p^{+}} \Bigl[B_{K}1_{\{B_{K} < \ceil{K/2}\}} \Bigr] &= \textstyle  \sum_{i=1}^{\ceil{K/2}-1}i {K \choose i}p^{i} (1-p)^{K-i} \\
&= \textstyle Kp\sum_{i=1}^{\ceil{K/2}-1}{K-1 \choose i-1}p^{i-1} (1-p)^{(K-1)-(i-1)} \\ 
&= \textstyle  Kp \mathbb{P}_{p^{+}}(B_{K-1} \leq \ceil{K/2} - 2).
\end{align*}
Substituting the latter quantity into \eqref{eq:bias-1} and using $Bias_{p}(\bar{p}(B_{K})):= p-\mathbb{E}_{p^{+}}[\bar{p}(B_{K})]$ yields \eqref{eq:bias_p}.

\end{proof}

\subsection{Empirical Results for Exponential and Cubic Test Functions}
\label{sub:ExponentialCubicResults}

Tables~\ref{tab:sampling-comp-mc-h2} and \ref{tab:sampling-comp-mc-h3} show four G-PBA sampling policies  implemented using six estimation schemes for $\hat{p}(\cdot)$ and batch size  for both $h_{2}$ and $h_{2}$ test functions, respectively. For each of these combinations, three performance statistics are computed (columns): absolute residuals, $\hat{r}_{K}^{\eta}(f_{T})$, length of the 95\% CI, $\hat{l}_{K,0.95}^{\eta}(f_{T})$; as well as its coverage $\hat{c}_{K,0.95}^{\eta}(f_{T})$, using 1,000 MC macro-replicates. Furthermore, the TPO policy is also included with $\alpha \in \{0.05,0.40\}$  (last rows). 

\begin{table}[H]
	\footnotesize
	\caption{MC performance statistics (expressed in $1\times 10^{-2}$ units) for G-PBAs for the test function $h_{2}$.}
	\label{tab:sampling-comp-mc-h2}
	\centering
	\begin{tabular}{c|c|rrr|rrr|rrr}
		\hline
		\multirow{ 2 }{*}{$\eta$} & \multirow{ 2 }{*}{$\hat{p}$} & \multicolumn{3}{c|}{$\hat{r}_{K}^{\eta}(f_{T})$ ($10^{-2}$)}& \multicolumn{3}{c|}{$\hat{l}_{K,0.95}^{\eta}(f_{T})$ ($10^{-2}$)}& \multicolumn{3}{c}{$\hat{c}_{K,0.95}^{\eta}(f_{T})$ (in \%)}\\
		& & $K=50$ & $K=250$ & $K=500$ & $K=50$ & $K=250$ & $K=500$ & $K=50$ & $K=250$ & $K=500$  \\
		\hline\parbox[t]{1mm}{\multirow{6}{*}{\rotatebox[origin=c]{90}{Det-IDS}}} & \multirow{ 1 }{*}{ $\bar{p}$ } & 0.2897 & 0.1554 & 0.1893 & 0.0034 & 0.0754 & 0.5755 & 0.20 & 12.00 & 70.80 \\
		& \multirow{ 1 }{*}{ $\hat{p}_{\mathscr{L}_{0}}$ } & 0.2953 & 0.1674 & 0.1827 & 0.0047 & 0.1068 & 0.6875 & 0.10 & 18.30 & 78.90 \\
		& \multirow{ 1 }{*}{ $\hat{p}_{\mathscr{L}_{1}}$ } & 0.2813 & 0.1628 & 0.1970 & 0.0019 & 0.0758 & 0.5631 & 0.30 & 11.60 & 68.90 \\
		& \multirow{ 1 }{*}{ $\hat{p}_{\mathscr{L}_{2}}$ } & 0.2903 & 0.1587 & 0.1997 & 0.0011 & 0.0631 & 0.5462 & 0.10 & 8.80 & 70.50 \\
		& \multirow{ 1 }{*}{ $\mathscr{P}_{\mathscr{M}}$ } & 0.2402 & 0.1305 & 0.1616 & 0.0020 & 0.1069 & 0.4433 & 0.20 & 20.10 & 67.60 \\
		& \multirow{ 1 }{*}{ $\mathscr{P}_{\mathscr{S}}$ } & 0.1657 & 0.1020 & 0.1229 & 0.0013 & 0.1175 & 0.4526 & 0.10 & 29.80 & 80.60 \\
		\cline{1-11}\parbox[t]{1mm}{\multirow{6}{*}{\rotatebox[origin=c]{90}{Rand-IDS}}} & \multirow{ 1 }{*}{ $\bar{p}$ } & 0.2816 & 0.1562 & 0.1891 & 0.0045 & 0.1096 & 0.6098 & 0.20 & 15.60 & 63.50 \\
		& \multirow{ 1 }{*}{ $\hat{p}_{\mathscr{L}_{0}}$ } & 0.2672 & 0.1713 & 0.1958 & 0.0013 & 0.1257 & 0.7220 & 0.00 & 19.50 & 69.60 \\
		& \multirow{ 1 }{*}{ $\hat{p}_{\mathscr{L}_{1}}$ } & 0.2683 & 0.1595 & 0.1871 & 0.0051 & 0.0868 & 0.6382 & 0.30 & 12.30 & 62.30 \\
		& \multirow{ 1 }{*}{ $\hat{p}_{\mathscr{L}_{2}}$ } & 0.2823 & 0.1618 & 0.1908 & 0.0011 & 0.0770 & 0.5862 & 0.00 & 12.00 & 59.80 \\
		& \multirow{ 1 }{*}{ $\mathscr{P}_{\mathscr{M}}$ } & 0.2694 & 0.1464 & 0.1617 & 0.0041 & 0.1201 & 0.4917 & 0.20 & 16.70 & 58.30 \\
		& \multirow{ 1 }{*}{ $\mathscr{P}_{\mathscr{S}}$ } & 0.1555 & 0.1044 & 0.1259 & 0.0087 & 0.1460 & 0.4928 & 1.10 & 32.60 & 71.20 \\
		\cline{1-11}\parbox[t]{1mm}{\multirow{6}{*}{\rotatebox[origin=c]{90}{Rand-Q}}} & \multirow{ 1 }{*}{ $\bar{p}$ } & 0.2413 & 0.1394 & 0.1256 & 0.0042 & 0.0362 & 0.1627 & 0.20 & 6.70 & 29.10 \\
		& \multirow{ 1 }{*}{ $\hat{p}_{\mathscr{L}_{0}}$ } & 0.2541 & 0.1491 & 0.1463 & 0.0050 & 0.0648 & 0.2719 & 0.30 & 9.70 & 42.10 \\
		& \multirow{ 1 }{*}{ $\hat{p}_{\mathscr{L}_{1}}$ } & 0.2524 & 0.1295 & 0.1234 & 0.0018 & 0.0304 & 0.1795 & 0.10 & 6.30 & 30.20 \\
		& \multirow{ 1 }{*}{ $\hat{p}_{\mathscr{L}_{2}}$ } & 0.2472 & 0.1319 & 0.1242 & 0.0025 & 0.0273 & 0.1378 & 0.20 & 4.30 & 25.90 \\
		& \multirow{ 1 }{*}{ $\mathscr{P}_{\mathscr{M}}$ } & 0.2785 & 0.1279 & 0.1261 & 0.0017 & 0.0583 & 0.2132 & 0.10 & 10.40 & 34.40 \\
		& \multirow{ 1 }{*}{ $\mathscr{P}_{\mathscr{S}}$ } & 0.1455 & 0.0942 & 0.0920 & 0.0014 & 0.0874 & 0.2609 & 0.30 & 22.30 & 58.00 \\
		\cline{1-11}\parbox[t]{1mm}{\multirow{6}{*}{\rotatebox[origin=c]{90}{Syst-Q}}} & \multirow{ 1 }{*}{ $\bar{p}$ } & 0.2752 & 0.1401 & 0.1236 & 0.0001 & 0.0303 & 0.1360 & 0.00 & 4.20 & 24.00 \\
		& \multirow{ 1 }{*}{ $\hat{p}_{\mathscr{L}_{0}}$ } & 0.2902 & 0.1539 & 0.1335 & 0.0042 & 0.0384 & 0.2055 & 0.10 & 5.60 & 32.60 \\
		& \multirow{ 1 }{*}{ $\hat{p}_{\mathscr{L}_{1}}$ } & 0.2704 & 0.1401 & 0.1247 & 0.0052 & 0.0309 & 0.1416 & 0.30 & 4.40 & 22.90 \\
		& \multirow{ 1 }{*}{ $\hat{p}_{\mathscr{L}_{2}}$ } & 0.2741 & 0.1384 & 0.1220 & 0.0019 & 0.0215 & 0.0998 & 0.20 & 2.10 & 18.00 \\
		& \multirow{ 1 }{*}{ $\mathscr{P}_{\mathscr{M}}$ } & 0.3067 & 0.1394 & 0.1170 & 0.0001 & 0.0251 & 0.1568 & 0.00 & 3.10 & 24.00 \\
		& \multirow{ 1 }{*}{ $\mathscr{P}_{\mathscr{S}}$ } & 0.1614 & 0.0905 & 0.0889 & 0.0011 & 0.0508 & 0.2043 & 0.20 & 11.80 & 46.40 \\
		\hline
		\parbox[t]{1mm}{\multirow{2}{*}{\rotatebox[origin=c]{90}{TPO}}} & $\tilde{p}_{0.05}$ & \multicolumn{3}{c|}{  0.6873}& \multicolumn{3}{c|}{ 5.2490}& \multicolumn{3}{c}{26.9}\\
		& $\tilde{p}_{0.40}$ & \multicolumn{3}{c|}{ 1.8503}& \multicolumn{3}{c|}{ 34.9012}& \multicolumn{3}{c}{87.5}\\
		\hline
	\end{tabular}
\end{table}

\begin{table}[htb]
	\footnotesize
	\caption{MC performance statistics (expressed in $1\times 10^{-2}$ units)  for the test function $h_{3}$. }
	\label{tab:sampling-comp-mc-h3}
	\centering
	\begin{tabular}{c|c|rr|rr|rr}
		\hline
		\multirow{ 2 }{*}{$\eta$} & \multirow{ 2 }{*}{$\hat{p}$} & \multicolumn{2}{c|}{$\hat{r}_{K}^{\eta}(f_{T})$ ($10^{-2}$)}& \multicolumn{2}{c|}{$\hat{l}_{K,0.95}^{\eta}(f_{T})$ ($10^{-2}$)}& \multicolumn{2}{c}{$\hat{c}_{K,0.95}^{\eta}(f_{T})$ (in \%)}\\
		&  & $K=250$ & $K=500$ & $K=250$ & $K=500$  & $K=250$ & $K=500$  \\
		\cline{1-8} \parbox[t]{1mm}{\multirow{6}{*}{\rotatebox[origin=c]{90}{Det-IDS}}}  & \multirow{ 1 }{*}{ $\bar{p}$ } & 5.0529 & 4.7236 & 0.4160 & 2.8673 & 2.10 & 16.80 \\
		& \multirow{ 1 }{*}{ $\hat{p}_{\mathscr{L}_{0}}$ } & 5.1432 & 4.9775 & 0.7467 & 3.5501 & 3.60 & 22.40 \\
		& \multirow{ 1 }{*}{ $\hat{p}_{\mathscr{L}_{1}}$ } & 5.0229 & 4.8121 & 0.3879 & 2.7467 & 2.00 & 17.90 \\
		& \multirow{ 1 }{*}{ $\hat{p}_{\mathscr{L}_{2}}$ } & 5.0077 & 4.8966 & 0.3384 & 2.4650 & 1.70 & 14.00 \\
		& \multirow{ 1 }{*}{ $\mathscr{P}_{\mathscr{M}}$ } & 4.7156 & 4.6418 & 0.9552 & 4.4379 & 6.30 & 28.50 \\
		& \multirow{ 1 }{*}{ $\mathscr{P}_{\mathscr{S}}$ } & 4.2941 & 4.1632 & 1.4502 & 5.8494 & 9.50 & 39.90 \\
		\cline{1-8} \parbox[t]{1mm}{\multirow{6}{*}{\rotatebox[origin=c]{90}{Rand-IDS}}}  & \multirow{ 1 }{*}{ $\bar{p}$ } & 4.8830 & 4.7830 & 0.6196 & 3.8335 & 3.60 & 20.30 \\
		& \multirow{ 1 }{*}{ $\hat{p}_{\mathscr{L}_{0}}$ } & 5.0229 & 4.8439 & 1.0181 & 5.0343 & 5.20 & 27.30 \\
		& \multirow{ 1 }{*}{ $\hat{p}_{\mathscr{L}_{1}}$ } & 4.9485 & 4.7307 & 0.5929 & 3.4080 & 2.80 & 19.20 \\
		& \multirow{ 1 }{*}{ $\hat{p}_{\mathscr{L}_{2}}$ } & 5.0483 & 4.7806 & 0.4014 & 3.5384 & 2.00 & 19.10 \\
		& \multirow{ 1 }{*}{ $\mathscr{P}_{\mathscr{M}}$ } & 4.8148 & 4.7400 & 1.1672 & 4.8150 & 6.00 & 26.70 \\
		& \multirow{ 1 }{*}{ $\mathscr{P}_{\mathscr{S}}$ } & 4.2613 & 4.2927 & 2.1320 & 6.8282 & 12.90 & 40.20 \\
		\cline{1-8} \parbox[t]{1mm}{\multirow{6}{*}{\rotatebox[origin=c]{90}{Rand-Q}}}  & \multirow{ 1 }{*}{ $\bar{p}$ } & 4.6976 & 4.6077 & 0.4251 & 2.4817 & 2.00 & 14.80 \\
		& \multirow{ 1 }{*}{ $\hat{p}_{\mathscr{L}_{0}}$ } & 4.7730 & 4.5869 & 0.6449 & 3.4914 & 3.60 & 20.10 \\
		& \multirow{ 1 }{*}{ $\hat{p}_{\mathscr{L}_{1}}$ } & 4.7619 & 4.4558 & 0.4583 & 1.8682 & 2.20 & 11.30 \\
		& \multirow{ 1 }{*}{ $\hat{p}_{\mathscr{L}_{2}}$ } & 4.7406 & 4.5710 & 0.2861 & 1.8134 & 1.70 & 10.60 \\
		& \multirow{ 1 }{*}{ $\mathscr{P}_{\mathscr{M}}$ } & 4.6666 & 4.5163 & 0.4972 & 2.6578 & 3.10 & 13.80 \\
		& \multirow{ 1 }{*}{ $\mathscr{P}_{\mathscr{S}}$ } & 4.0224 & 3.8971 & 1.3335 & 4.7615 & 8.50 & 32.40 \\
		\cline{1-8} \parbox[t]{1mm}{\multirow{6}{*}{\rotatebox[origin=c]{90}{Syst-Q}}}  & \multirow{ 1 }{*}{ $\bar{p}$ } & 4.8925 & 4.4606 & 0.1828 & 1.8556 & 0.70 & 10.10 \\
		& \multirow{ 1 }{*}{ $\hat{p}_{\mathscr{L}_{0}}$ } & 4.9710 & 4.7632 & 0.3906 & 2.6305 & 2.30 & 14.40 \\
		& \multirow{ 1 }{*}{ $\hat{p}_{\mathscr{L}_{1}}$ } & 4.9345 & 4.4922 & 0.2637 & 1.4268 & 1.30 & 7.90 \\
		& \multirow{ 1 }{*}{ $\hat{p}_{\mathscr{L}_{2}}$ } & 4.9454 & 4.6179 & 0.2593 & 1.0887 & 1.00 & 6.60 \\
		& \multirow{ 1 }{*}{ $\mathscr{P}_{\mathscr{M}}$ } & 4.8032 & 4.5485 & 0.2690 & 1.7403 & 1.70 & 8.40 \\
		& \multirow{ 1 }{*}{ $\mathscr{P}_{\mathscr{S}}$ } & 4.0632 & 3.9204 & 0.7132 & 3.3071 & 4.10 & 21.80 \\
		\hline
		\parbox[t]{1mm}{\multirow{2}{*}{\rotatebox[origin=c]{90}{TPO}}} & $\tilde{p}_{0.05}$ & \multicolumn{2}{c|}{   5.1850}& \multicolumn{2}{c|}{ 45.5564}& \multicolumn{2}{c}{94.1}\\
		& $\tilde{p}_{0.40}$ & \multicolumn{2}{c|}{13.1249}& \multicolumn{2}{c|}{ 48.3949}& \multicolumn{2}{c}{55.0}\\
		\hline
	\end{tabular}
\end{table}

\bibliography{mybibArxiv}{}
\bibliographystyle{siam}

\end{document}